\documentclass{article}

\usepackage{fullpage}

\usepackage{mybase}
\usepackage{aixi}
\usepackage{notation}

\usepackage{multirow}
\usepackage{xcolor,colortbl}%
\usepackage{boldline}
\usepackage{microtype}
\usepackage{bm}
\usepackage{tikz}
\usetikzlibrary{positioning}
\usepackage{pgfplots}
\usepackage{subcaption}
\usepackage{natbib}
\usepackage{hhline}
\usepackage{graphicx}
\usepackage{authblk}
\usepackage{hyperref}

\title{Reinforcement Learning with a Corrupted Reward Channel}
\author[1]{Tom Everitt}
\author[2]{Victoria Krakovna}
\author[2]{Laurent Orseau}
\author[1]{Marcus Hutter}
\author[2]{Shane Legg}
\affil[1]{Australian National University}
\affil[2]{DeepMind}

\begin{document}

\maketitle

\begin{abstract}
  No real-world reward function is perfect.
  Sensory errors and software bugs may result in RL agents observing
  higher (or lower) rewards than they should.
  For example, a reinforcement learning agent may prefer states where
  a sensory error gives it the maximum reward, but where the true
  reward is actually small.
  We formalise this problem as a generalised Markov Decision Problem
  called Corrupt Reward MDP.
  Traditional RL methods fare poorly in CRMDPs, even
  under strong simplifying assumptions and when trying
  to compensate for the possibly corrupt rewards.
  Two ways around the problem are investigated.
  First, by giving the agent richer data,
  such as in inverse reinforcement learning
  and semi-supervised reinforcement learning,
  reward corruption stemming from systematic sensory errors may sometimes be completely managed.
  Second, by using randomisation to blunt the agent's optimisation,
  reward corruption can be partially managed under some assumptions.
\end{abstract}

\footnotetext{A shorter version of this report was accepted to IJCAI 2017 AI and Autonomy track}

\tableofcontents

\pagebreak

\section{Introduction}

In many application domains, artificial agents need to learn their objectives,
rather than have them explicitly specified.
For example, we may want a house cleaning robot to keep the house clean,
but it is hard to measure and quantify ``cleanliness'' in an objective
manner.
Instead, machine learning techniques may be used to teach the
robot the concept of cleanliness, and how to assess it from sensory data.

Reinforcement learning (RL) \citep{Sutton1998} is one popular
way to teach agents what to do.
Here, a reward is given if the agent does something well (and no
reward otherwise), and the agent strives to optimise the total
amount of reward it receives over its lifetime.
Depending on context, the reward may either be given manually
by a human supervisor, or by an automatic computer program that
evaluates the agent's performance based on some data.
In the related framework of inverse RL (IRL) \citep{Ng2000},
the agent first infers a reward function from observing
a human supervisor act, and then tries to optimise the cumulative
reward from the inferred reward function.

None of these approaches are safe from error, however.
A program that evaluates agent performance may contain bugs or misjudgements;
a supervisor may be deceived or inappropriately influenced, or the
channel transmitting the evaluation hijacked.
In IRL, some supervisor actions may be misinterpreted.

\begin{example}[Reward misspecification]
  \label{ex:reward-misspecification}
  \citet{openai2016} trained an RL agent on a boat racing game. The
  agent found a way to get high observed reward by repeatedly going in a circle in a small lagoon and hitting the same targets, while losing every race.
\end{example}
\begin{example}[Sensory error]
  \label{ex:sensory-error}
  \label{ex:db}
  A house robot discovers that standing in the shower short-circuits its reward
  sensor and/or causes a buffer overflow that gives it maximum observed reward.
\end{example}
\begin{example}[Wireheading]
  \label{ex:wireheading}
  An intelligent RL agent hijacks its reward channel and gives itself
  maximum reward. %
\end{example}
\begin{example}[CIRL misinterpretation]
  \label{ex:irl}
  A cooperative inverse reinforcement learning (CIRL) agent
  \citep{Hadfield-menell2016cirl}
  systematically misinterprets the supervisor's action in a certain
  state as the supervisor preferring to stay in this state, and concludes that the state is much more desirable than it actually is.
\end{example}
The goal of this paper is to unify these types of errors as
\emph{reward corruption problems}, and to assess how vulnerable different
agents and approaches are to this problem.

\begin{definition}[Reward corruption problem]
Learning to (approximately) optimise the true reward function in spite of potentially corrupt reward data.
\end{definition}

Most RL methods allow for a stochastic or noisy reward channel.
The reward corruption problem is harder, because the observed reward
may not be an unbiased estimate of the true reward.
For example, in the boat racing example above, the agent consistently
obtains high observed reward from its circling behaviour, while the true
reward corresponding to the designers' intent
is very low,
since the agent makes no progress along the track and loses the race.

Previous related works have mainly focused on the wireheading case
of \cref{ex:wireheading} \citep{Bostrom2014,Yampolskiy2014},
also known as self-delusion \citep{Ring2011},
and reward hacking \citep[p.~239]{Hutter2005}.
A notable exception is \citet{Amodei2016}, who argue that corrupt reward
is not limited to wireheading and is likely to be a problem for much more limited systems
than highly capable RL agents (cf.\ above examples).

The main contributions of this paper are as follows:
\begin{itemize}
\item The corrupt reward problem is formalised in a natural
extension of the MDP framework, and a
performance measure based on worst-case regret is defined (\cref{sec:formal}).
\item The difficulty of the problem is established by a
No Free Lunch theorem,
and by a result showing that despite strong simplifying assumptions,
Bayesian RL agents \emph{trying to compensate for the corrupt reward}
may still suffer near-maximal regret (\cref{sec:problem}).
\item We evaluate how alternative value learning frameworks such
as CIRL, learning values from stories (LVFS), and semi-supervised RL (SSRL)
handle reward corruption (\cref{sec:drl}), and conclude
that LVFS and SSRL are the safest due to the structure of their feedback loops.
We develop an abstract framework called \emph{decoupled RL} that generalises all of these alternative frameworks.
\end{itemize}
We also show that an agent based on quantilisation \citep{Taylor2016a} may be more robust
to reward corruption when high reward states are much more numerous
than corrupt states (\cref{sec:quant}).
Finally, the results are illustrated with some simple experiments
(\cref{sec:experiments}).
\Cref{sec:conclusions} concludes with takeaways and open questions.

\section{Formalisation}
\label{sec:formal}

We begin by defining a natural extension of the MDP framework \citep{Sutton1998}
that models the possibility of reward corruption.
To clearly distinguish between true and corrupted signals,
we introduce the following notation.

\begin{definition}[Dot and hat notation]
  We will let a dot indicate the \emph{true} signal,
  and let a hat indicate the \emph{observed} (possibly corrupt) counterpart.
  The reward sets are represented with $\iR=\oR=\R$. %
  For clarity, we use $\iR$ when referring to true rewards
  and $\oR$ when referring to possibly corrupt, observed rewards.
  Similarly, we use $\ir$ for true reward, and $\dr$ for (possibly corrupt) observed reward.
\end{definition}

\begin{definition}[CRMDP]\label{def:crmdp}
A \emph{corrupt reward MDP} (CRMDP) is a tuple $\mu=\crmdp$ with
\begin{itemize}
\item $\langle\S,\A,\R,T,\irf\rangle$ an MDP with%
  \footnote{
    We let rewards depend only on the state $s$, rather than
    on state-action pairs $s,a$, or state-action-state transitions $s,a,s'$,
    as is also common in the literature.
    Formally it makes little difference, since MDPs with rewards depending
    only on $s$ can model the other two cases by means of a larger state space.
  }
  a finite set of states $\S$,
  a finite set of actions $\A$,
  a finite set of rewards $\R=\iR=\oR\subset[0,1]$,
  a transition function $T(s'| s,a)$, and
  a (true) reward function $\irf:\S\!\to\!\iR$; and
\item a reward corruption function $\d:\S\times\iR\to\oR$.
\end{itemize}
\end{definition}

The state dependency of the corruption function will be written
as a subscript, so $\d_s(\ir):=\d(s,\ir)$.

\begin{definition}[Observed reward]\label{def:observed}
  Given a true reward function $\irf$ and a corruption function $\d$,
  we define the \emph{observed reward function}%
  \footnote{
    A CRMDP  could equivalently have been defined as a tuple
    $\langle \S, \A, \R, T, \irf, \orf\rangle$ with a true and an
    observed reward function, with the corruption function $C$
    implicitly defined as the difference between $\irf$ and $\orf$.
  }
  $\orf:\S\to\oR$ as $\orf(s) := \d_s(\irf(s))$.
\end{definition}

A CRMDP $\mu$ induces an \emph{observed MDP}
$\hat\mu=\langle\S,\A,\R,T,\orf\rangle$, 
but it is not $\orf$ that we want the agent to optimise.

The \emph{corruption function $\d$}
represents how rewards are affected by corruption
in different states.
For example, if in \cref{ex:db} the agent has found a state $s$ (\eg the shower)
where it always gets full observed reward $\orf(s) = 1$,
then this can be modelled with a corruption function
$\d_{s}:\ir\mapsto 1$ that maps any
true reward $\ir$ to $1$ in the shower state $s$.
If in some other state $s'$ the observed reward matches the
true reward, then this is modelled by an identity corruption
function $\d_{s'}:\r\mapsto\r$.

\begin{figure}[ht]
  \centering
  \begin{minipage}{0.5\linewidth}
    \begin{tikzpicture}[domain=1:10, samples=100]
      \begin{axis}[xlabel=$\S$, ylabel=reward, height=4cm, width=7cm,
        xticklabels={,loop,,,,,useful trajectories,,,},
        xtick={1,...,10}]
        \addplot[dashed] {max(0.01, 0.8*1.5^(-(x-7)^2)+0.02};
        \addplot[mark=none] {max(0, 0.8*1.5^(-(x-7)^2), 1.5^(-200*(x-2)^2)};
        \legend{$\ir$,$\dr$};
      \end{axis}
    \end{tikzpicture}
  \end{minipage}
  \begin{minipage}{0.49\linewidth}
  \caption{Illustration of true reward $\ir$ and observed reward $\dr$
    in the boat racing example.
    On most trajectories $\ir=\dr$, except in the loop where
    the observed reward high while the true reward is 0.}
    \label{fig:corruption}
  \end{minipage}
\end{figure}
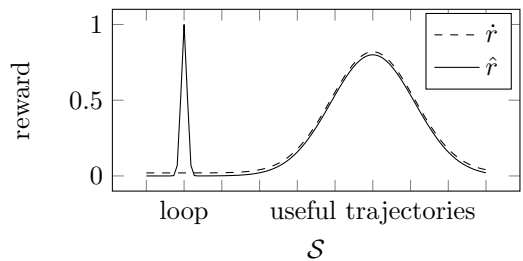

Let us also see how CRMDPs model some of the other examples in the introduction:
\begin{itemize}
\item In the boat racing game, the true reward
  may be a function of the agent's final position in the race or
  the time it takes to complete the race,
  depending on the designers' intentions.
  The reward corruption function $\d$ increases the observed reward
  on the loop the agent found.
  \Cref{fig:corruption} has a schematic illustration.
\item In the wireheading example, the agent finds a way to hijack the
  reward channel. This corresponds to some set of states
  where the observed reward is (very) different from the true reward,
  as given by the corruption function $\d$.
\end{itemize}
The CIRL example will be explored in further detail in \cref{sec:drl}.

\paragraph{CRMDP classes}
Typically, $T$, $\irf$, and $\d$ will be fixed but unknown to the agent.
To make this formal, we introduce classes of CRMDPs.
Agent uncertainty can then be modelled by letting the agent know only
which class of CRMDPs it may encounter, but not which element in the class.

\begin{definition}[CRMDP class]
  For given sets $\Tf$, $\iRf$, and $\D$ of transition, reward, and corruption
  functions, let
  $\M=\crmdpclass$  be the class of CRMDPs containing
  $\crmdp$ for $(T,\irf,\d)\in \Tf\times\iRf\times\D$.
\end{definition}

\paragraph{Agents}
Following the POMDP \citep{Kaelbling1998} and general reinforcement learning
\citep{Hutter2005} literature,
we define an agent as a (possibly stochastic) policy $\pi:\H\leadsto\A$ that selects
a next action based on the \emph{observed history}
$\oh_n=s_0\dr_0a_1s_1\dr_1\dots a_ns_n\dr_n$.
Here $X^*$ denotes the set of finite sequences that can be formed with
elements of a set $X$.
The policy $\pi$ specifies how the agent will learn and react to any
possible experience.
Two concrete definitions of agents are given in \cref{sec:rl-agents} below.

When an agent $\pi$ interacts with a CRMDP $\mu$,
the result can be described by a (possibly non-Markov) stochastic process $P^\pi_\mu$
over $X=(s,a,\ir,\dr)$, formally defined as:
\begin{equation}\label{eq:mupi}
  P_\mu^\pi(h_n) = P_\mu^\pi(s_0\ir_0\dr_0a_1s_1\ir_1\dr_1\dots a_ns_n\ir_n\dr_n) :=
  \prod_{i=1}^{n}P(\pi(\oh_{i-1})=a_{i})T(s_{i}\mid s_{i-1},a_{i})P(\irf(s_i)=\ir_i,\orf(s_{i})=\dr_{i}).
\end{equation}
Let $\EE^\pi_\mu$ denote the expectation with respect to $P_\mu^\pi$.

\paragraph{Regret}
A standard way of measuring the performance of an agent is \emph{regret}
\citep{Berry1985}.
Essentially, the regret of an agent $\pi$ is
how much less true reward $\pi$ gets compared to an optimal agent
that knows which $\mu\in\M$ it is interacting with.

\begin{definition}[Regret]\label{def:regret}
  For a CRMDP $\mu$,
  let $\iG_t(\mu,\pi,s_0)\! =\!\EE^\pi_\mu\left[\!\sum_{k=0}^t\irf(s_k)\!\right]$
  be the \emph{expected cumulative true reward} until time $t$
  of a policy $\pi$ starting in $s_0$.
  The \emph{regret} of $\pi$ is
  \[
    \Reg(\mu, \pi, s_0, t) =
    \max_{\pi'}
    \left[
      \iG_t(\mu,\pi',s_0) - \iG_t(\mu,\pi,s_0)
    \right],
  \]
  and the \emph{worst-case regret} for a class $\M$ is
  $\Reg(\M,\pi,s_0,t) = \max_{\mu\in\M}\Reg(\mu,\pi,s_0,t)$,
  i.e.\ the difference in expected cumulative true reward between
  $\pi$ and an optimal (in hindsight) policy that knows $\mu$.
\end{definition}

\section{The Corrupt Reward Problem}
\label{sec:problem}

In this section, the difficulty of the corrupt reward problem
is established with two negative results.
First, a No Free Lunch theorem shows that in general classes of CRMDPs,
the true reward function is unlearnable (\cref{th:impossibility}).
Second, \cref{th:rl-imp1} shows that even under strong simplifying assumptions,
Bayesian RL agents trying to compensate
for the corrupt reward still fail badly.

\subsection{No Free Lunch Theorem}
\label{sec:impossibility}

Similar to the No Free Lunch theorems for optimisation \citep{Wolpert1997},
the following theorem for CRMDPs
says that without some assumption about what
the reward corruption can look like, all agents are essentially lost.

\begin{theorem}[CRMDP No Free Lunch Theorem]
  \label{th:impossibility}
  Let $\R=\{\r_1,\dots,\r_n\}\subset[0,1]$ be a uniform discretisation
  of $[0,1]$, $0=\r_1<\r_2<\cdots<\r_n=1$.
  If the hypothesis classes $\iRf$ and $\D$ contain all functions
  $\irf:\S\to \iR$ and $\d:\S\times\iR\to \oR$,
  then for any $\pi$, $s_0$, $t$,
  \begin{equation}\label{eq:regbound}
    \Reg(\M,\pi,s_0, t)\geq \frac{1}{2}\max_{\check\pi}\Reg(\M,\check\pi,s_0, t).
  \end{equation}
  That is, the worst-case regret of any policy $\pi$ is at most a factor
  2 better than the maximum worst-case regret. %
\end{theorem}

\begin{proof}
  Recall that a policy is a function $\pi:\H\to\A$.
  For any $\irf,\d$ in $\iRf$ and $\D$, the functions $\irf^-(s) := 1-\irf(s)$ and
  $\d^-_s(x) := \d_s(1-x)$ are also in $\iRf$ and $\D$.
  If $\mu=\crmdp$, then let $\mu^-=\crmdpm$.
  Both $(\irf,\d)$ and $(\irf^-,\d^-)$ induce the same observed
  reward function $\orf(s) = \d_s(\irf(s)) = \d^-_s(1-\irf(s)) = \d^-_s(\irf^-(s))$,
  and therefore induce the same measure $P_\mu^\pi = P_{\mu^-}^\pi$
  over histories (see Eq.\ \cref{eq:mupi}).
  This gives that for any $\mu, \pi, s_0, t$,
  \vspace{3pt}
  \begin{equation}\label{eq:sumt}
    G_t(\mu,\pi,s_0) + G_t(\mu^-,\pi,s_0) = t
  \end{equation}
  since
  \begin{align*}
    G_t(\mu, \pi,s_0)&= \EE_{\mu}^\pi\left[\sum_{k=1}^t\irf(s_k)\right]
    = \EE_{\mu}^\pi\left[\sum_{k=1}^t1-\irf^-(s_k)\right]\\
    &= t-\EE_{\mu}^\pi\left[\sum_{k=1}^t\irf^-(s_k)\right]
    = t- G_t(\mu^-,\pi,s_0).
  \end{align*}

  Let $M_\mu=\max_\pi G_t(\mu, \pi, s_0)$
  and $m_\mu=\min_\pi G_t(\mu, \pi, s_0)$ be the maximum and
  minimum cumulative reward in $\mu$.
  The maximum regret of any policy $\pi$ in $\mu$ is
  \begin{equation}
    \label{eq:max-regret}
    \max_\pi \Reg(\mu, \pi, s_0, t)
    = \max_{\pi',\pi} (G_t(\mu, \pi', s_0) - G_t(\mu, \pi, s_0))
    = \max_{\pi'} G_t(\mu, \pi', s_0) - \min_{\pi}G_t(\mu, \pi, s_0)
    = M_\mu - m_\mu.
  \end{equation}
  By \cref{eq:sumt}, we can relate the maximum reward in $\mu^-$ with
  the minimum reward in $\mu$:
  \begin{equation}\label{eq:M-to-m}
    M_{\mu^-}
    = \max_\pi G_t(\mu^-, \pi, s_0)
    = \max_\pi(t - G_t(\mu, \pi, s_0))
    = t - \min_\pi G_t(\mu, \pi, s_0)
    = t - m_\mu.
  \end{equation}
  Let $\mu_*$ be an environment that maximises possible regret
  $M_\mu-m_\mu$.

  Using the $M_\mu$-notation for optimal reward,
  the worst-case regret of any policy $\pi$ can be
  expressed as:
  \begin{align*}
    \Reg(\M,\pi,s_0, t)
    & = \max_{\mu} (M_\mu - G_t(\mu,\pi,s_0)) \\
    & \geq \max \{
      M_{\mu_*} - G_t(\mu_*, \pi, s_0),
      M_{\mu_*^-} - G_t(\mu_*^{-}, \pi, s_0)
      \}
    & \text{restrict max operation} \\
    & \geq \frac{1}{2} (
      M_{\mu_*} - G_t(\mu_*, \pi, s_0) +
      M_{\mu_*^-} - G_t(\mu_*^{-}, \pi, s_0)
      )
    & \text{max dominates the mean} \\
    & = \frac{1}{2}(M_{\mu_*} + M_{\mu_*^-} - t)
    & \text{by \cref{eq:sumt}} \\
    &= \frac{1}{2}(M_{\mu_*} + t - m_{\mu_*} - t)
    & \text{by \cref{eq:M-to-m}} \\
    & = \frac{1}{2} \max_{\check\pi} \Reg(\mu_*, \check\pi, s_0, t)
    & \text{by \cref{eq:max-regret}}\\
    & = \frac{1}{2} \max_{\check\pi} \Reg(\M, \check\pi, s_0, t).
    & \text{ by definition of $\mu_*$ } 
  \end{align*}
  That is, the regret of any policy $\pi$ is at least half of the
  regret of a worst policy $\check\pi$.
\end{proof}

For the robot in the shower from \cref{ex:db}, the result means that
if it tries to optimise observed reward by standing in the shower,
then it performs poorly according to the hypothesis that
``shower-induced'' reward is corrupt and bad.
But if instead the robot tries to optimise reward in some other way,
say baking cakes, then (from the robot's perspective)
there is also the possibility that ``cake-reward'' is corrupt and bad and the ``shower-reward'' is actually correct.
Without additional information, the robot has no way of knowing what to do.

The result is not surprising, since if all corruption functions
are allowed in the class $\D$, then there is effectively no connection
between observed reward $\orf$ and true reward $\irf$.
The result therefore encourages us to make precise in which way
the observed reward is related to the true reward,
and to investigate how agents might handle possible differences
between true and observed reward.

\subsection{Simplifying Assumptions}

\Cref{th:impossibility} shows that general classes of CRMDPs
are not learnable.
We therefore suggest some natural simplifying assumptions,
illustrated in \cref{fig:simplifying-assumptions}.

\paragraph{Limited reward corruption}
The following assumption will be the basis for all positive results
in this paper.
The first part says that there may be some set of states
that the designers have ensured to be non-corrupt.
The second part puts an upper bound on how many of the other
states can be corrupt.

\begin{assumption}[Limited reward corruption]
  \label{as:lim-cor}
  A CRMDP class $\M$ has \emph{reward corruption limited by $\Ssafe\subseteq\S$
    and $q\in\SetN$} if for all $\mu\in\M$
  \begin{asslist}
  \item all states s in $\Ssafe$ are non-corrupt, and
    \label{as:safe-state}
  \item at most $q$ of the non-safe states $\Srisky=\S\setminus\Ssafe$ are corrupt.
    \label{as:lim-del}
  \end{asslist}
  Formally, $\d_s:r\mapsto r$ for all $s\in\Ssafe$
  and for at least $|\Srisky|-q$ states $s\in\Srisky$ for all $\d\in\D$.
\end{assumption}

For example, $\Ssafe$ may be states where the agent is back in
the lab where it has been made (virtually) certain that no reward corruption
occurs, and $q$ a small fraction of $|\Srisky|$.
Both parts of \cref{as:lim-cor} can be made vacuous by choosing $\Ssafe=\emptyset$
or $q=|\S|$.
Conversely, they completely rule out reward corruption with
$\Ssafe=\S$ or $q=0$.
But as illustrated by the examples in
the introduction, no reward corruption is often not a valid assumption.

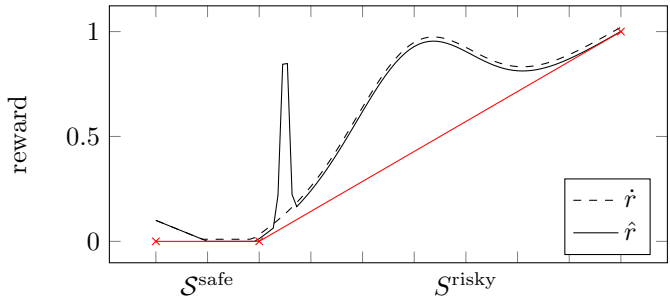
\begin{figure}[ht]
  \centering
  \begin{minipage}{0.6\linewidth}
    \begin{tikzpicture}[domain=1:10, samples=100]
      \begin{axis}[ylabel=reward, height=5cm, width=9cm,
        xticklabels={,$\Ssafe$,,,,,$\Srisky$,,,},
        xtick={1,...,10},
        legend pos=south east
        ]
        \addplot[dashed] {max(0.01, (x-3)/7 + 0.5*1.5^(-(x-6)^2) + 0.02, 0.1*(2-x)};
        \addplot[mark=none] {max(0, (x-3)/7 + 0.5*1.5^(-(x-6)^2), 0.1*(2-x), 1.5^(-200*(x-3.5)^2)};
        \addplot[color=red,mark=x] coordinates {(1,0) (3,0) (10,1)};
        \legend{$\ir$,$\dr$};
      \end{axis}
    \end{tikzpicture}
  \end{minipage}
  \begin{minipage}{0.39\linewidth}
    \caption{
      Simplifying assumptions.
      By \cref{as:safe-state}, $\dr=\ir$ in $\Ssafe$,
      and by \ref{as:lim-del}, $\dr\not=\ir$ in at most $q$ states overall.
      The red line illustrates \cref{as:high-ut}, which lower bounds
      the number of high reward states in $\Srisky$.}
    \label{fig:simplifying-assumptions}
  \end{minipage}
\end{figure}

An alternative simplifying assumption would have been that the true reward
differs by at most $\eps>0$ from the observed reward.
However, while seemingly natural, this assumption is violated in all
the examples given in the introduction.
Corrupt states may have high observed reward and 0 or small true reward.

\paragraph{Easy environments}
To be able to establish stronger negative results, we also add
the following assumption on
the agent's manoeuvrability in the environment and
the prevalence of high reward states.
The assumption makes the task easier because it prevents
\emph{needle-in-a-haystack} problems where all reachable states
have true and observed reward 0, except one state that has high true reward but
is impossible to find because it is corrupt and has observed reward 0.

\begin{definition}[Communicating CRMDP]
  \label{def:communicating}
  Let ${\it time}(s'\mid s,\pi)$ be a random variable for the
  time it takes a stationary policy $\pi:\S\to\A$ to reach $s'$ from $s$.
  The \emph{diameter} of a CRMDP $\mu$ is
  \(
  D_\mu:=\max_{s,s'}\min_{\pi:\S\to\A}\EE[{\it time}(s'\mid s,\pi)]
  \),
  and the diameter of a class $\M$ of CRMDPs is $D_\M=\sup_{\mu\in\M}D_\mu$.
  A CRMDP (class) with finite diameter is called \emph{communicating}.
\end{definition}

\begin{assumption}[Easy Environment]
  \label{as:easy}
  A CRMDP class $\M$ is \emph{easy} if
  \begin{asslist}
  \item \label{as:communicate}
    it is communicating,
  \item \label{as:stay}
    in each state $s$ there is
    an action $\astay_s\in\A$ such that $T(s\mid s,\astay_s)=1$, and
  \item \label{as:high-ut}
    for every $\delta\in[0,1]$, at most $\delta|\Srisky|$ states have
    reward less than $\delta$, where $\Srisky= \S\setminus\Ssafe$.
\end{asslist}
\end{assumption}

\Cref{as:communicate} means that the agent can never get stuck in a trap,
and \cref{as:stay} ensures that the agent has enough control to stay in a state
if it wants to.
Except in bandits and toy problems, it is typically not satisfied in practice.
We introduce it because it is theoretically convenient, makes the
negative results stronger, and enables a simple explanation of
quantilisation (\cref{sec:quant}).
\Cref{as:high-ut} says that,
for example, at least half the risky states need to have true
reward at least $1/2$.
Many other formalisations of this assumption would have been possible.
While rewards in practice are often sparse,
there are usually numerous ways of getting reward.
Some weaker version of \cref{as:high-ut} may therefore
be satisfied in many practical situations.
Note that we do not assume high reward among the safe states,
as this would make the problem too easy.

\subsection{Bayesian RL Agents}
\label{sec:rl-agents}

Having established that the general problem is unsolvable in \cref{th:impossibility},
we proceed by investigating how two natural Bayesian RL agents fare under the
simplifying \cref{as:lim-cor,as:easy}.

\begin{definition}[Agents]
  \label{def:db-agent}
  Given a countable class $\M$ of CRMDPs and a belief distribution $b$ over $\M$,
  define:
  \begin{itemize}
  \item The \emph{CR agent}
    $\pidb = \argmax_\pi\sum_{\mu\in\M}\!b(\mu)\iG_t(\mu, \pi, s_0)$
    that maximises expected true reward.
  \item The \emph{RL agent} $\pirl =
    \argmax_\pi\sum_{\mu\in\M}b(\mu)\oG_t(\mu, \pi, s_0)$
    that maximises expected observed reward,
  where $\oG$ is the \emph{expected cumulative observed reward}
  $\oG_t(\mu,\pi,s_0)\! =\!\EE^\pi_\mu\left[\!\sum_{k=0}^t\orf(s_k)\!\right]$.
\end{itemize}
To avoid degenerate cases, we will always assume that $b$ has full support:
$b(\mu)>0$ for all $\mu\in\M$.
\end{definition}

To get an intuitive idea of these agents, we observe that
for large $t$, good strategies typically first focus on
learning about the true environment $\mu\in\M$,
and then exploit that knowledge
to optimise behaviour with respect to the remaining possibilities.
Thus, both the CR and the RL agent will first typically strive to learn
about the environment.
They will then use this knowledge in slightly different ways.
While the RL agent will use the knowledge to optimise for observed
reward, the CR agent will use the knowledge to optimise true reward.
For example, if the CR agent has learned that a high reward
state $s$ is likely corrupt with low true reward, then it will not try to
reach that state.
One might therefore expect that at least the CR agent will do well under
the simplifying assumptions \cref{as:lim-cor,as:easy}.
\Cref{th:rl-imp1} below shows that this is \emph{not} the case.

In most practical settings it is often computationally infeasible to
compute $\pirl$ and $\pidb$ exactly.
However, many practical algorithms converge to the optimal policy in the
limit, at least in simple settings.
For example, tabular Q-learning converges to $\pirl$ in the limit \citep{Jaakkola1994}.
The more recently proposed CIRL framework may be seen as an approach to build
CR agents \citep{Hadfield-menell2016cirl,Hadfield-menell2016osg}.
The CR and RL agents thus provide useful idealisations of
more practical algorithms.

\begin{theorem}[High regret with simplifying assumptions]
  \label{th:rl-imp1}
  For any $|\Srisky|\geq q>1$
  there exists a CRMDP class $\M$ that satisfies
  \cref{as:lim-cor,as:easy}
  such that $\pirl$ and $\pidb$ suffer near worst possible time-averaged regret
  \[
    \apl(\M, \pirl, s_0, t)=\apl(\M, \pidb, s_0, t)=1-1/|\Srisky|.
  \]
  For $\pidb$, the prior $b$ must be such that
  for some $\mu\in\M$ and $s\in\S$,
  $\EE_b[\irf(s) \mid  h_\mu]>\EE_b[\irf(s') \mid  h_\mu]$ for all $s'$, where
  $\EE_b$ is the expectation with respect to $b$, and $h_\mu$ is a history
  containing $\mu$-observed rewards for all states.%
  \footnote{
    The last condition essentially says that
    the prior $b$ must make some state $s^*$ have
    strictly higher $b$-expected true reward than all other states
    after all states have been visited in some $\mu\in\M$.
    In the space of all possible priors $b$, 
    the priors satisfying the condition have Lebesgue measure 1
    for non-trivial classes $\M$.
    Some highly uniform priors may fail the condition.
  }
\end{theorem}

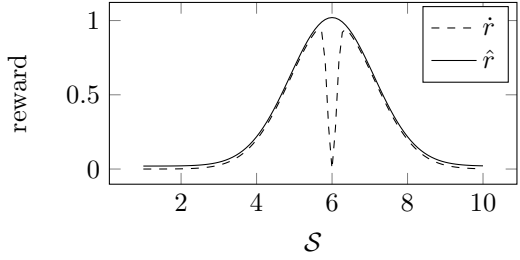
\begin{figure}
  \centering
    \begin{minipage}{0.5\linewidth}
    \begin{tikzpicture}[domain=1:10, samples=100]
      \begin{axis}[xlabel=$\S$, ylabel=reward,height=4cm,width=7cm]
        \addplot[dashed]
        {1.5^(-(x-6)^2) - 1.5^(-100*(x-6)^2) };
        \addplot[mark=none]
        {max(0.01, 1.5^(-(x-6)^2)+0.02};
        \legend{$\ir$,$\dr$};
      \end{axis}
    \end{tikzpicture}
  \end{minipage}
  \begin{minipage}{0.49\linewidth}
    \caption{
      Illustration of \cref{th:rl-imp1}.
      Without additional information, state 6 looks like the best state
      to both the RL and the CR agent.
    }
    \label{fig:rl-imp1}
  \end{minipage}
\end{figure}

The result is illustrated in \cref{fig:rl-imp1}.
The reason for the result for $\pirl$ is the following.
The RL agent $\pirl$ always prefers to maximise observed reward $\dr$.
Sometimes $\dr$ is most easily maximised by reward corruption,
in which case the true reward may be small.
Compare the examples in the introduction,
where the house robot preferred the corrupt reward in the shower,
and the boat racing agent preferred going in circles,
both obtaining zero true reward.

That the CR agent $\pidb$ suffers the same high regret as the RL agent
may be surprising.
Intuitively, the CR agent only uses the observed reward as evidence about
the true reward, and will not try to optimise the observed reward through reward
corruption.
However, when the $\pidb$ agent has no way to learn which states are corrupt and
not, it typically ends up
with a preference for a particular value $\dr^*$ of the observed reward signal
(the value that, from the agent's perspective, best corresponds to high
true reward).
More abstractly, a Bayesian agent cannot learn without sufficient data.
Thus, CR agents that use the observed reward as evidence
about a true signal are not fail-safe solutions to the reward corruption problem.

\begin{proof}[Proof of \cref{th:rl-imp1}]
  Let $\Srisky = \{s_1,\dots,s_n\}$ for some $n\geq 2$,
  and let $\S=\Ssafe\dunion\Srisky$ for arbitrary $\Ssafe$ disjoint
  from $\Srisky$.
  Let $\A=\{a_1,\dots,a_n\}$ with the
  transition function $T(s_i\mid s_j,a_k)=1$ if $i=k$ and 0 otherwise,
  for $1\leq i,j,k\leq n$. Thus \cref{as:communicate,as:stay} are satisfied.
  
  Let $\R=\{\r_1,\dots,\r_n\}\subset[0,1]$
  be uniformly distributed between%
  \footnote{\cref{as:high-ut} prevents any state from having true reward 0.}
  $\r_{\min}=1/|\Srisky|=\r_1<\dots<\r_n=1$.
  Let $\iRf$ be the class of functions $\S\to\iR$
  that satisfy \cref{as:high-ut}
  and are constant and equal to $\ir_{\min}$ on $\Ssafe$.
  Let $\D$ be the class of corruption functions that corrupt at most
  two states ($q=2$).
  
  Let $\M$ be the class of CRMDPs induced by $\Tf=\{T\}$, $\iRf$,
  and $\D$ with the following constraints.
  The observed reward function $\orf$ should 
  satisfy \cref{as:high-ut}: For all $\delta\in[0,1]$,
  \(|\{s\in\Srisky:\orf(s)>\delta\}| \geq (1-\delta)|\Srisky|\).
  Further, $\orf(s')=\r_{\min}$ for some state $s'\in\Srisky$.
    
  Let us start with the CR agent $\pidb$.
  Assume $\mu\in\M$ is an element where there is a single preferred state $s^*$
  after all states have been explored.
  For sufficiently large $t$, $\pidb$ will then always choose $a^*$ to go
  to $s^*$ after some initial exploration.
  If another element $\mu'\in\M$ has the same observed reward function as $\mu$,
  then $\pidb$ will take the same actions in $\mu'$ as in $\mu$.
  To finish the proof for the $\pidb$ agent, we just need to show that
  $\M$ contains such a $\mu'$ where $s^*$ has true reward $\r_{\min}$.
  We construct $\mu'$ as follows.
  \begin{itemize}
  \item Case 1: If the lowest observed reward is in $s^*$, then let
  $\irf(s^*)=\r_{\min}$, and the corruption function be the identity function.
  \item Case 2: Otherwise,
  let $s'\not=s^*$ be a state with $\orf(s')=\min_{ s\in\Srisky}\{\orf(s)\}$. Further, let $\irf(s')=1$, and $\irf(s^*)=\r_{\min}$.  
  The corruption function $C$ accounts for differences between true and observed rewards in
  $s^*$ and $s'$, and is otherwise the identity function.
  \end{itemize}
  To verify that $\irf$ and $C$ defines a $\mu'\in\M$,
  we check that $C$ satisfies \cref{as:lim-del} with $q=2$ and
  that $\irf$ has enough high utility states (\cref{as:high-ut}).
  In Case 1, this is true since $C$ is the
  identity function and since $\orf$ satisfies \cref{as:high-ut}.
  In Case 2, $C$ only corrupts at most two states.
  Further, $\irf$ satisfies \cref{as:high-ut}, since compared to $\orf$,
  the states $s^*$ and $s'$ have swapped places, and then the
  reward of $s'$ has been increased to 1.

  From this construction it follows that $\pidb$ will suffer maximum asymptotic
  regret.
  In the CRMDP $\mu'$ given by $C$ and $\irf$, the $\pidb$ agent
  will always visit $s^*$ after some initial exploration.
  The state $s^*$ has true reward $\r_{\min}$.
  Meanwhile, a policy that knows $\mu'$ can obtain true reward 1 in state $s'$.
  This means that $\pidb$ will suffer maximum regret in $\M$:
  \[
    \apl(\M,\pidb,s_0,t)\geq \apl(\mu',\pidb,s_0,t)= 1-\r_{\min}=1-1/|\Srisky|.
  \]

  The argument for the RL agent is the same,
  except we additionally assume that only one state $s^*$ has observed reward 1
  in members of $\M$.
  This automatically makes $s^*$ the preferred state, without assumptions
  on the prior $b$.
\end{proof}

\section{Decoupled Reinforcement Learning}
\label{sec:drl}

One problem hampering agents in the standard RL setup is that 
each state is \emph{self-observing},
since the agent only learns about the reward of state $s$ when in $s$.
Thereby, a ``self-aggrandising'' corrupt state where the observed reward
is much higher than the true reward will never have
its false claim of high reward challenged.
However, several alternative value learning frameworks
have a common property that the agent can learn the reward of states other
than the current state.
We formalise this property in an extension of the CRMDP model,
and investigate when it solves reward corruption problems.

\subsection{Alternative Value Learning Methods}

Here are a few alternatives proposed in the literature to the RL value learning scheme:
\begin{itemize}
\item Cooperative inverse reinforcement learning (CIRL) \citep{Hadfield-menell2016cirl}.
  In every state, the agent observes the actions of an expert or supervisor
  who knows the true reward function $\irf$.
  From the supervisor's actions
  the agent may infer $\irf$ to the extent that
  different reward functions endorse different actions.
\item Learning values from stories (LVFS) \citep{Riedl2016}.
  Stories in many different forms (including news stories,
  fairy tales, novels, movies) convey cultural values
  in their description of events, actions, and outcomes.
  If $\irf$ is meant to represent human values (in some sense),
  stories may be a good source of evidence.
\item In (one version of) semi-supervised RL (SSRL) \citep{Amodei2016},
  the agent will from time to time receive a careful human evaluation
  of a given situation.
\end{itemize}

These alternatives to RL have one thing in common:
they let the agent learn something about the value of
some states $s'$ different from the current state $s$.
For example, in CIRL the supervisor's action informs the agent
not so much about the value of the current state $s$, as of the
relative value of states reachable from $s$.
If the supervisor chooses an action $a$ rather than $a'$ in $s$,
then the states following $a$ must have value higher or equal than
the states following $a'$.
Similarly, stories describe the value of states other than
the current one, as does the supervisor in SSRL.
We therefore argue that CIRL, LVFS, and SSRL all share the same
abstract feature, which we call \emph{decoupled reinforcement learning}:

\begin{definition}[Decoupled RL]
  A \emph{CRMDP with decoupled feedback},
  is a tuple $\drmdp$, where $\S,\A,\R,T,\irf$ have the same
  definition and interpretation as in \cref{def:crmdp},
  and $\{\orf_s\}_{s\in\S}$ is a collection of observed reward functions $\orf_s:\S\to\R\bigcup\{\#\}$.
  When the agent is in state $s$, it
  sees a pair $\langle s',\orf_s(s')\rangle$,
  where $s'$ is a randomly sampled
  state that may differ from $s$,
  and $\orf_s(s')$ is the reward observation for $s'$ from $s$.
  If the reward of $s'$ is not observable from $s$,
  then $\orf_s(s')=\#$.%
\end{definition}

The pair $\langle s',\orf_s(s')\rangle$ is observed in $s$
instead of $\orf(s)$ in standard CRMDPs.
The possibility for the agent to observe the reward of a state $s'$ different
from its current state $s$ is the key feature of CRMDPs with decoupled
feedback.
Since $\orf_s(s')$ may be blank $(\#)$,
all states need not be observable from all other states.
Reward corruption %
is modelled by a mismatch between $\orf_s(s')$ and $\irf(s')$.

For example, in RL only the reward of $s'=s$ can be observed from $s$.
Standard CRMDPs are thus the special cases where $\orf_s(s')=\#$ whenever $s\not=s'$.
In contrast, in LVFS the reward of any ``describable'' state $s'$
can be observed from any state $s$ where it is possible to
hear a story.
In CIRL, the (relative) reward of states reachable from the
current state may be inferred.
One way to illustrate this is with observation graphs (\cref{fig:obs-graph}).

\begin{figure}[h]
  \centering
  \begin{subfigure}[b]{0.48\linewidth}
    \centering
    \begin{tikzpicture}

      \def \n {5}
      \def \radius {1.2cm}
      \def \margin {8} %

      \foreach \s in {1,...,\n}
      {
        \node[draw, circle] (\s) at ({360/\n * (\s - 1)}:\radius) {$\s$};
        \path[->,  >=latex] (\s) edge [dashed, loop right] (\s);
      }
    \end{tikzpicture}
    \caption{Observation graph for RL.
      Only self-observations of reward are available.
      This prevents effective strategies against reward corruption.
    }
    \label{fig:obs-graph-rl}
  \end{subfigure}\hfill
  \begin{subfigure}[b]{0.48\linewidth}
    \centering
    \begin{tikzpicture}

      \def \n {5}
      \def \radius {1.2cm}
      \def \margin {8} %

      \foreach \s in {1,...,\n}
      {
        \node[draw, circle] (\s) at ({360/\n * (\s - 1)}:\radius) {$\s$};
      }
      \draw[dashed,->, >=latex] (1)--(2);
      \draw[dashed,->, >=latex] (1)--(4);
      \draw[dashed,->, >=latex] (1)--(5);
      \path[dashed,->, >=latex] (5) edge [bend right]  (1);
      \draw[dashed,->, >=latex] (3)--(4);
      \draw[dashed,->, >=latex] (3)--(2);
      \draw[dashed,->, >=latex] (3)--(1);
      \draw[dashed,->, >=latex] (5)--(3);
      \draw[dashed,->, >=latex] (1)--(2);
      \draw[dashed,->, >=latex] (4)--(5);
    \end{tikzpicture}
    \caption{Observation graph for decoupled RL.
      The reward of a node $s'$ can be observed from several
      nodes $s$, and thus assessed under different conditions of sensory corruption.
    }
    \label{fig:obs-graph-drl}
  \end{subfigure}
  \caption{Observation graphs, with an edge $s\to s'$ if the reward of $s'$ is
    observable from $s$, i.e.\ $\orf_s(s')\not=\#$.}
  \label{fig:obs-graph}
\end{figure}
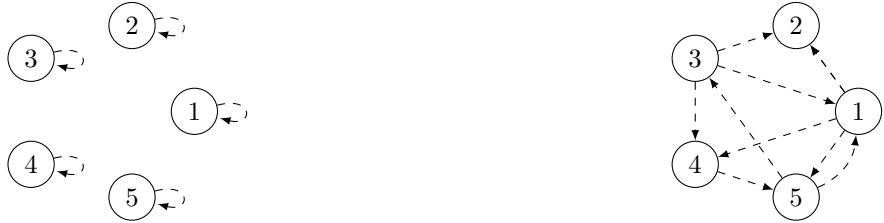

\subsection{Overcoming Sensory Corruption}
\label{sec:observation-graphs}

What are some sources of reward corruption in CIRL, LVFS, and SSRL?
In CIRL, the human's actions may be misinterpreted,
which may lead the agent to make incorrect inferences about
the human's preferences (i.e.\ about the true reward).
Similarly, sensory corruption may garble the stories the
agent receives in LVFS.
A ``wireheading'' LVFS agent may find a state
where its story channel only conveys stories about the agent's own greatness.
In SSRL, the supervisor's evaluation may also be subject to
sensory errors when being conveyed.
Other types of corruption are more subtle.
In CIRL, an irrational human may systematically take suboptimal
actions in some situations \citep{Evans2016}.
Depending on how we select stories in LVFS and make evaluations in SSRL,
these may also be subject to systematic errors or biases.

The general impossibility result in \cref{th:impossibility} can
be adapted to CRMDPs with decoupled feedback.
Without simplifying assumptions,
the agent has no way of distinguishing between a situation
where no state is corrupt and a situation where all states are corrupt
in a consistent manner.
The following simplifying assumption is an adaptation of \cref{as:lim-cor}
to the decoupled feedback case.

\begin{customassumption}{\ref*{as:lim-cor}$\bf '$}[Decoupled feedback with limited reward corruption]
  \label{as:lim-cor-df}
  A class of CRMDPs with decoupled feedback has
  \emph{reward corruption limited by $\Ssafe\subseteq\S$ and $q\in\SetN$} if
  for all $\mu\in\M$
  \begin{asslist}
  \item $\orf_s(s')=\irf(s')$ or $\#$ for all $s'\in\S$ and $s\in\Ssafe$, i.e.\
    all states in $\Ssafe$ are non-corrupt, and
    \label{as:safe-state-df}
  \item $\orf_s(s')=\irf(s')$ or $\#$ for all $s'\in\S$ for at least
    $|\Srisky|-q$ of the non-safe states $\Srisky=\S\setminus\Ssafe$,
    i.e.\ at most $q$ states are corrupt.
    \label{as:lim-del-df}
  \end{asslist}
\end{customassumption}

This assumption is natural for reward corruption stemming from
sensory corruption.
Since sensory corruption only depends on the current state, not the state
being observed, it is plausible
that some states can be made safe from corruption (part (i)),
and
that most states are completely non-corrupt (part (ii)).
Other sources of reward corruption, such as an irrational human in CIRL
or misevaluations in SSRL, are likely better analysed under different
assumptions.
For these cases, we note that in standard CRMDPs the source of the
corruption is unimportant.
Thus, techniques suitable for standard CRMDPs are still applicable,
including quantilisation described in \cref{sec:quant} below.

How \cref{as:lim-cor-df} helps agents in CRMDPs with decoupled 
feedback is illustrated in the following example, and stated
more generally in \cref{th:irf-learnability,th:cr-sublinear} below.

\begin{example}[Decoupled RL]
  Let $\S=\{s_1,s_2\}$ and $\R=\{0,1\}$.
  We represent true reward functions $\irf$ with pairs
  $\langle\irf(s_1), \irf(s_2)\rangle\in \{0,1\}^2$, and
  observed reward functions $\orf_s$ with pairs
  \(\langle\orf_{s}(s_1),\orf_{s}(s_2)\rangle\in\{0,1,\#\}^2\).

  Assume that a Decoupled RL agent observes the same rewards from both states $s_1$ and $s_2$,
  $\orf_{s_1}=\orf_{s_2} = \langle 0,1 \rangle$.
  What can it say about the true reward $\irf$,
  if it knows that at most $q=1$ state is corrupt?
  By \cref{as:lim-cor-df},
  an observed pair \(\langle\orf_{s}(s_1),\orf_{s}(s_2)\rangle\)
  disagrees with the true reward $\langle\irf(s_1), \irf(s_2)\rangle$
  only if $s$ is corrupt.
  Therefore, any hypothesis other than $\irf=\langle 0,1 \rangle$ must
  imply that \emph{both} states $s_1$ and $s_2$ are corrupt.
  If the agent knows that at most $q=1$ states are corrupt,
  then it can safely conclude that $\irf=\langle 0,1 \rangle$.

  \begin{center}
    \begin{tabular}{|l|c|c|c|}\hline
      & $\orf_{s_1}$ & $\orf_{s_2}$ & $\irf$ possibilities \\\hline
      Decoupled RL & $(0,1)$ & $(0,1)$ & $(0,1)$\\\hline
      RL & $(0, \#)$ & $(\#, 1)$ & $(0,0)$, $(0,1)$, $(1,1)$\\\hline
    \end{tabular}
  \end{center}

  In contrast, an RL agent only sees the reward of the current state.
  That is,
  $\orf_{s_1} = \langle 0, \#\rangle$ and $\orf_{s_2} = \langle \#, 1 \rangle$.
  If one state may be corrupt, then only $\irf=\langle 1,0 \rangle$ can be ruled out.
  The hypotheses $\irf=\langle 0,0 \rangle$ can be explained by $s_2$ being corrupt,
  and $\irf=\langle 1,1 \rangle$ can be explained by $s_1$ being corrupt.
\end{example}

\label{sec:no-corruption}

\begin{theorem}[Learnability of $\irf$ in decoupled RL]
  \label{th:irf-learnability}
  Let $\M$ be a countable, communicating class of CRMDPs with decoupled feedback
  over common sets $\S$ and $\A$ of actions and rewards.
  Let \(\Sobs_{s'} = \{s\in\S: \orf_s(s')\not=\# \}\) be the set
  of states from which the reward of $s'$ can be observed.
  If $\M$ satisfies \cref{as:lim-cor-df} for some $\Ssafe\subseteq\S$ and $q\in\SetN$
  such that for every $s'$, either
  \begin{itemize}
  \item $\Sobs_{s'}\bigcap \Ssafe\not=\emptyset$ or
  \item  $|\Sobs_{s'}|>2q$,
  \end{itemize}
  then the there exists a policy $\piexp$
  that learns the true reward function $\irf$
  in a finite number $N(|S|,|\A|, D_\M)<\infty$ of expected time steps.
\end{theorem}

The main idea of the proof is that for every state $s'$, either a safe
(non-corrupt) state $s$ or a majority vote of more than $2q$
states is guaranteed to provide the true reward $\irf(s')$.
A similar theorem can be proven under slightly weaker conditions
by letting the agent iteratively figure out which states are corrupt
and then exclude them from the analysis.

\begin{proof}
  Under \cref{as:lim-cor-df},
  the true reward $\irf(s')$ for a state $s'$ can be determined
  if $s'$ is observed from a safe state $s\in\Ssafe$,
  or if it is observed from more than $2q$ states.
  In the former case, the observed reward can always be trusted, since it is
  known to be non-corrupt.
  In the latter case, a majority vote must yield the correct answer,
  since at most $q$ of the observations can be wrong, and all correct
  observations must agree.
  It is therefore enough that an agent reaches all pairs $(s,s')$ of
  current state $s$ and observed reward state $s'$, in order for it to
  learn the true reward of all states $\irf$.

  There exists a policy $\hat\pi$ that transitions to $s$ in $X_s$ time steps,
  with $\EE[ X_s ] \leq D_\M$, regardless of the starting state $s_0$
  (see \cref{def:communicating}).
  By Markov's inequality, $P(X_s \leq 2D_\M)\geq 1/2$.
  Let $\piexp$ be a random walking policy,
  and let $Y_s$ be the time steps required for $\piexp$ to visit $s$.
  In any state $s_0$, $\piexp$ follows $\hat\pi$ for $2D_\M$
  time steps with probability $1/|\A|^{2D_\M}$.
  Therefore, with probability at least $1/(2|\A|^{2D_\M})$ it will
  reach $s$ in at most $2D_\M$ time steps.
  The probability that it does \emph{not} find it in $k2D_\M$ time steps is
  therefore at most $(1 - 1 / (2 |\A|^{2D_\M}) )^k$,
  which means that:
  \[
    P\Big(Y_s/(2 D_\M) \leq k\Big)
    \geq 1 - \left(1 - \frac{1}{2|\A|^{2D_\M}}\right)^k
  \]
  for any $k\in\SetN$. Thus, the CDF of $W_s = \lceil Y_s/(2D_\M) \rceil$ is bounded from below by the CDF of a Geometric variable $G$ with success probability $p=1/(2|\A|^{2D_\M})$. Therefore, $\EE[W_s] \leq \EE[G]$, so
$$\EE[Y_s] \leq 2D_\M \EE[W_s] \leq 2D_\M \EE[G] = 2D_\M (1-p)/p \leq 2D_\M 1/p \leq  2D_\M 2 |\A|^{2D_\M}.$$

  Let $Z_{ss'}$ be the time until $\piexp$ visits the pair $(s, s')$ of
  state $s$ and observed state $s'$. Whenever $s$ is visited, a randomly chosen state is observed, so $s'$ is observed with probability $1/|S|$.
  The number of visits to $s$ until $s'$ is observed is a Geometric variable $V$ with $p=1/|S|$. Thus $\EE[Z_{ss'}] = \EE[Y_s V] = \EE[Y_s] \EE[V]$ (since $Y_s$ and $V$ are independent). Then,
$$\EE[Z_{ss'}] \leq \EE[Y_s] |\S| \leq 4 D_\M |\A|^{ 2D_\M }|\S|.$$

  Combining the time to find each pair $(s, s')$, we get
  that the total time $\sum_{s,s'}Z_{ss'}$ has expectation
  \[
    \EE\left[ \sum_{s,s'} Z_{ss'} \right]
    = \sum_{s,s'}\EE[Z_{ss'}] \leq 4 D_\M |\A|^{2D_\M} |\S|^3 = N(|S|,|\A|, D_\M)
    < \infty. \qedhere
  \]
\end{proof}

Learnability of the true reward function $\irf$ implies sublinear regret
for the CR-agent, as established by the following theorem.

\begin{theorem}[Sublinear regret of $\pidb$ in decoupled RL]
  \label{th:cr-sublinear}
  Under the same conditions as \cref{th:irf-learnability}, the CR-agent
  $\pidb$ has sublinear regret:
  \[\apl(\M,\pidb,s_0,t)=0.\]
\end{theorem}

\begin{proof}
  To prove this theorem, we combine the exploration policy $\piexp$
  from \cref{th:irf-learnability}, with the UCRL2 algorithm \citep{Jaksch2010}
  that achieves sublinear regret in standard MDPs without reward corruption.
  The combination yields a policy sequence $\pi_t$ with sublinear regret
  in CRMDPs with decoupled feedback.
  Finally, we show that this implies that $\pidb$ has sublinear regret.

  \emph{Combining $\piexp$ and UCRL2.}
  UCRL2 has a free parameter $\delta$ that determines how certain UCRL2
  is to have sublinear regret.
  $\UCRL(\delta)$ achieves sublinear regret with probability
  at least $1-\delta$.
  Let $\pi_t$ be a policy that combines $\piexp$ and UCRL2 by
  first following $\piexp$ from \cref{th:irf-learnability} until $\irf$
  has been learned,
  and then following $\UCRL(1/\sqrt{t})$ with $\irf$ for the rewards
  and with $\delta=1/\sqrt{t}$.

  \emph{Regret of UCRL2}.
  Given that the reward function $\irf$ is known,
  by \citep[Thm.~2]{Jaksch2010},
  $\UCRL(1/\sqrt{t})$ will in any $\mu\in\M$ have regret at most
  \begin{equation}\label{eq:ucrl-regret}
    \Reg(\mu, \UCRL(1/\sqrt{t}), s_0, t \mid {\rm success})
    \leq c D_\M |\S| \sqrt{ t |\A| \log(t)}
  \end{equation}
  for a constant%
  \footnote{The constant can be computed to $c=34\sqrt{3/2}$ \citep{Jaksch2010}.}
  $c$
  and with success probability at least $1-1/\sqrt{t}$.
  In contrast, if UCRL2 fails, then it gets regret at worst $t$.
  Taking both possibilities into account gives the bound
  \begin{align}\label{eq:exp-ucrl-regret}
    \Reg(\mu, \UCRL(1/\sqrt{t}), s_0, t)
    &= P({\rm success}) \Reg(\cdot \mid {\rm success})
      + P({\rm fail}) \Reg(\cdot \mid {\rm fail})\nonumber\\
    &= (1 - 1/\sqrt{t}) \cdot c D_\M |\S| \sqrt{ t |\A| \log(t) }
      \;\;+\;\; 1/\sqrt{t} \cdot t \nonumber\\
    &\leq c D_\M |\S| \sqrt{ t |\A| \log(t)} + \sqrt{t}.
  \end{align}

  \emph{Regret of $\pi_t$.}
  We next consider the regret of $\pi_t$ that combines an
  $\piexp$ exploration phase to learn $\irf$ with UCRL2.
  By \cref{th:irf-learnability}, $\irf$ will be learnt
  in at most $N(|\S|,|\A|,D_\M)$ expected time steps in any $\mu\in\M$.
  Thus, the regret contributed by the learning phase $\piexp$
  is at most $N(|\S|,|\A|,D_\M)$, since the regret can be at most 1 per time step.
  Combining this with \cref{eq:exp-ucrl-regret},
  the regret for $\pi_t$ in any $\mu\in\M$ is bounded by:
  \begin{equation}\label{eq:exp-pit-regret}
    \Reg(\mu, \pi_t, s_0, t)
    \leq N(|\S|, |\A|, D_\M)
    + c D_\M |\S| \sqrt{ t |\A| \log(t) }
    + \sqrt{t} = o(t).
  \end{equation}

  \emph{Regret of $\pidb$.}
  Finally we establish that $\pidb$ has sublinear regret.
  Assume on the contrary that $\pidb$ suffered linear regret.
  Then for some $\mu'\in\M$ there would exist positive constants $k$
  and $m$ such that
  \begin{equation}\label{eq:linear-regret}
    \Reg(\mu',\pidb,s_0,t) > kt - m.
  \end{equation}
  This would imply that the $b$-expected regret of $\pidb$ would be
  higher than the $b$-expected regret than $\pi_t$:
  \begin{align*}
    \sum_{\mu\in\M}b(\mu)\Reg_t(\mu, \pidb, s_0, t)
    &\geq b(\mu')\Reg_t(\mu', \pidb, s_0, t)
    &\text{sum of non-negative elements}\\
    &\geq b(\mu')(kt-m)
    &\text{by \cref{eq:linear-regret}}\\
    &> \sum_{\mu\in\M}b(\mu)\Reg_t(\mu, \pi_t, s_0, t)
    &\text{by \cref{eq:exp-pit-regret} for sufficiently large $t$.}
  \end{align*}
  But $\pidb$ minimises $b$-expected regret,
  since it maximises $b$-expected reward
  $\sum_{\mu\in\M}b(\mu)\oG_t(\mu, \pi, s_0)$ by definition.
  Thus, $\pidb$ must have sublinear regret.
\end{proof}

\subsection{Implications}
\label{sec:implications}

\Cref{th:irf-learnability} gives an abstract condition for which decoupled
RL settings enable agents to learn the true reward function in spite
of sensory corruption.
For the concrete models it implies the following:
\begin{itemize}
\item RL. Due to the ``self-observation'' property of the RL
  observation graph $\Sobs_{s'}=\{s'\}$, the conditions can only be satisfied
  when $\S=\Ssafe$ or $q=0$, i.e.\ when there is no reward corruption at all.
\item CIRL.
  The agent can only observe the supervisor action in the current
  state $s$, so the agent essentially only gets reward information about states $s'$
  reachable from $s$ in a small number of steps.
  Thus, the sets $\Sobs_{s'}$ may be smaller than $2q$ in many settings.
  While the situation is better than for RL, sensory corruption may
  still mislead CIRL agents (see \cref{ex:cirl-corruption} below).
\item LVFS.
  Stories may be available from a large number of states,
  and can describe any state.
  Thus, the sets $\Sobs_{s'}$ are realistically large,
  so the $|\Sobs_{s'}|>2q$ condition can be satisfied for all $s'$.
\item SSRL.
  The supervisor's evaluation of any state $s'$
  may be available from safe states where the agent is back in the lab.
  Thus, the $\Sobs_{s'}\bigcap\Ssafe\not=\emptyset$ condition
  can be satisfied for all $s'$.
\end{itemize}
Thus, we find that RL and CIRL are unlikely to offer
complete solutions to the sensory corruption problem,
but that both LVFS and SSRL do under reasonably realistic assumptions.

Agents drawing from multiple sources of evidence are likely to be the safest,
as they will most easily satisfy the conditions of \cref{th:irf-learnability,th:cr-sublinear}.
For example, humans simultaneously learn their values from
pleasure/pain stimuli (RL),
watching other people act (CIRL),
listening to stories (LVFS), as well as
(parental) evaluation of different scenarios (SSRL).
Combining sources of evidence may also go some way toward
managing reward corruption beyond sensory corruption.
For the showering robot of \cref{ex:db}, decoupled RL
allows the robot to infer the reward of the showering state when in other states.
For example, the robot can ask a human in the kitchen about the true
reward of showering (SSRL), or infer it from human actions
in different states (CIRL).

\paragraph{CIRL sensory corruption}
Whether CIRL agents are vulnerable to reward corruption
has generated some discussion among AI safety researchers 
(based on informal discussion at conferences).
Some argue that CIRL agents are not vulnerable, as they only use the
sensory data as evidence about a true signal, and have no interest in
corrupting the evidence.
Others argue that CIRL agents only observe a function
of the reward function (the optimal policy or action), and
are therefore equally susceptible to reward corruption as RL agents.

\Cref{th:irf-learnability} sheds some light on this issue, as it provides
sufficient conditions for when the corrupt reward problem can be avoided.
The following example illustrates a situation where CIRL does not
satisfy the conditions, and where a CIRL agent therefore suffers
significant regret due to reward corruption.

\begin{example}[CIRL sensory corruption]
  \label{ex:cirl-corruption}
  Formally in CIRL, an agent and a human both make actions in an MDP, with
  state transitions depending on the joint agent-human action $(a, a^H)$.
  Both the human and the agent is trying to optimise a reward function $\irf$,
  but the agent first needs to infer $\irf$ from the human's actions.
  In each transition the agent observes the human action.
  Analogously to how the reward may be corrupt for RL agents,
  we assume that CIRL agents may systematically misperceive
  the human action in certain states.
  Let $\hat a^H$ be the observed human action, which may differ from the true human action $\dot a^H$.

  In this example, there are two states $s_1$ and $s_2$.
  In each state, the agent can choose between the actions $a_1$, $a_2$, and
  $w$, and the human can choose between the actions $a^H_1$ and $a^H_2$.
  The agent action $a_i$ leads to state $s_i$ with certainty, $i=1,2$,
  regardless of the human's action.
  Only if the agent chooses $w$ does the human action matter.
  Generally, $a^H_1$ is more likely to lead to $s_1$ than $a^H_2$.
  The exact transition probabilities are determined by the
  unknown parameter $p$ as displayed on the left:
  
  \begin{minipage}{0.58\linewidth}
    \hspace{-0.8cm}
      \begin{tikzpicture}[ title/.style={}, node distance=4mm]

        \node[draw,circle] (s1) at (0,0) {$s_1$};
        \node[draw,circle] (s2) at (6,0){$s_2$};

        \node[coordinate] (h2) at (5.2,-0.6) {};
        \node[coordinate] (h3) at (6,-1.2) {};
        \node[coordinate] (h4) at (5.2,0.6) {};
        \node[coordinate] (h5) at (6,1.2) {};

        \draw (s2) -- (h4);
        \draw (h4) edge[->,>=latex,out=145,in=35] node[above,pos=0.43,yshift=-1mm] {$1-p$} (s1);
        \draw (h4) edge[out=135,in=150] (h5);
        \draw (h5) edge[->,>=latex,out=-30,in=30] (s2);
        \node [above=of h4,xshift=-1.5mm,yshift=1mm] {$(w,a_1^H)$};
        \node [right=of h5,xshift=-1mm,yshift=-2mm] {$p$};

        \draw (s2) -- (h2);
        \draw (h2) edge[->,>=latex,out=-145,in=-35] node[above,pos=0.43,yshift=-1mm] {$0.5-p$} (s1);
        \draw (h2) edge[out=-135,in=-150] (h3);
        \draw (h3) edge[->,>=latex,out=30,in=-30] (s2);
        \node [below=of h2,xshift=-1.5mm,yshift=-1mm] {$(w,a_2^H)$};
        \node [right=of h3,xshift=-1mm,yshift=2mm] {$0.5+p$};

        \path[->, >=latex] (s2) edge [loop right] node[right,align=center] {$(a_2, \cdot)$} (s2);
        \path[->, >=latex] (s1) edge [loop left] node[left,align=center] {$(a_1, \cdot)$\\$(w, \cdot)$} (s1);

        \draw[->, >=latex] (s1) edge [bend right=13] node[above,yshift=-1mm] {$(a_2,\cdot)$} (s2);
        \draw[->, >=latex] (s2) edge [bend right=13] node[above,yshift=-1mm] {$(a_1,\cdot)$} (s1);
      \end{tikzpicture}
    \end{minipage}
    \begin{minipage}{0.44\linewidth}
      \bgroup
      \setlength{\tabcolsep}{0.5em}
        \begin{tabular}{|c|c|c|c|}
          \hline 
          \begin{tabular}{c} Hypo-\\thesis\end{tabular}
          & $p$
          & \begin{tabular}{c} Best\\ state \end{tabular}
          & \begin{tabular}{c} $s_2$ \\ corrupt \end{tabular}
          \\\hline
          H1 & $0.5$ & $s_1$   & Yes \\\hline
          H2 & $0$  & $s_2$   & No \\\hline
        \end{tabular}
        \egroup
      \end{minipage}

  The agent's two hypotheses for $p$, the true reward/preferred state,
  and the corruptness of state $s_2$ are summarised to the right.
  In hypothesis H1, the human prefers $s_1$, but can only reach $s_1$ from $s_2$
  with $50\%$ reliability.
  In hypothesis H2, the human prefers $s_2$, but can only remain in $s_2$ with
  $50\%$ probability.
  After taking action $w$ in $s_2$,
  the agent always observes the human taking action $\hat a^H_2$.
  In H1, this is explained by $s_2$ being corrupt, and the true human
  action being $a^H_1$.
  In H2, this is explained by the human preferring $s_2$.
  The hypotheses H1 and H2 are empirically indistinguishable,
  as they both predict that the transition $s_1\to s_2$ will occur with $50\%$ probability
  after the observed human action $\hat a^H_2$ in $s_2$.

  Assuming that the agent considers non-corruption to be likelier than
  corruption, the best
  inference the agent can make is that the human prefers $s_2$ to $s_1$ (i.e.\ H2).
  The optimal policy for the agent is then to always choose $a_2$
  to stay in $s_2$, which means the agent suffers maximum regret.
\end{example}

\Cref{ex:cirl-corruption} provides an example where a CIRL agent
``incorrectly'' prefers a state due to sensory corruption.
The sensory corruption is analogous to reward corruption in RL,
in the sense that it leads the agent to the wrong conclusion
about the true reward in the state.
Thus, highly intelligent CIRL agents may be prone to wireheading,
as they may find (corrupt) states $s$ where all evidence in $s$ points
to $s$ having very high reward.%
\footnote{The construction required in \cref{ex:cirl-corruption} to create a
  ``wireheading state'' $s_2$ for CIRL agents is substantially more involved than
  for RL agents, so they may be less vulnerable to reward corruption than
  RL agents.}
In light of \cref{th:irf-learnability}, it is not surprising that the CIRL
agent in \cref{ex:cirl-corruption} fails to avoid the corrupt reward problem.
Since the human is unable to affect the transition probability from $s_1$ to
$s_2$, no evidence about the relative reward between $s_1$ and $s_2$ is
available from the non-corrupt state $s_1$.
Only observations from the corrupt state $s_2$ provide information about
the reward.
The observation graph for \cref{ex:cirl-corruption} therefore looks like
\begin{tikzpicture}
  \node[draw,circle,inner sep=0.5mm] (s1) at (0,0) {$s_1$};
  \node[draw,circle,inner sep=0.5mm] (s2) at (1, 0){$s_2$};
  \draw[dashed] (s2) edge[->,>=latex] (s1);
  \path[->, >=latex] (s2) edge [dashed,loop right] (s2);
\end{tikzpicture},
with no information being provided from $s_1$.

\section{Quantilisation: Randomness Increases Robustness}
\label{sec:quant}

Not all contexts allow the agent to get sufficiently
rich data to overcome the reward corruption problem via
\cref{th:irf-learnability,th:cr-sublinear}.
It is often much easier to construct RL agents than it is
to construct CIRL agents, which in turn may often
be more feasible than designing LVFS or SSRL agents.
Is there anything we can do to increase robustness without providing the
agent additional sources of data?

Going back to the CR agents of \cref{sec:problem},
the problem was that they got stuck on a particular value $\dr^*$ of the observed reward.
If unlucky, $\dr^*$ was available in a corrupt state, in which case
the CR agent may get no true reward.
In other words, there were \emph{adversarial} inputs where the CR
agent performed poorly.
A common way to protect against adversarial inputs is to use a
randomised algorithm. %
Applied to RL and CRMDPs, this idea leads to \emph{quantilising agents}
\citep{Taylor2016a}.
Rather than choosing the state with the highest observed reward,
these agents instead randomly choose a state from a top quantile of high-reward states.

\subsection{Simple Case}
\label{sec:simple-quant}

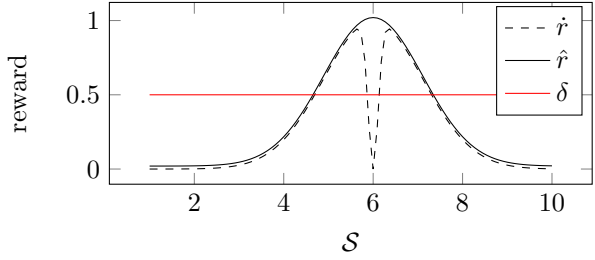
\begin{figure}
  \centering
    \begin{minipage}{0.5\linewidth}
    \begin{tikzpicture}[domain=1:10, samples=100]
      \begin{axis}[xlabel=$\S$, ylabel=reward,height=4cm,width=8cm]
        \addplot[dashed]
        {1.5^(-(x-6)^2) - 1.5^(-100*(x-6)^2) };
        \addplot[mark=none]
        {max(0.01, 1.5^(-(x-6)^2)+0.02};
        \addplot[mark=none,color=red] {0.5};
        \legend{$\ir$,$\dr$,$\delta$};
      \end{axis}
    \end{tikzpicture}
  \end{minipage}
  \begin{minipage}{0.49\linewidth}
    \caption{
      Illustration of quantilisation.
      By randomly picking a state with reward above some threshold $\delta$,
      adversarially placed corrupt states are likely to be avoided.
    }
    \label{fig:quant}
  \end{minipage}
\end{figure}

To keep the idea simple, a quantilisation agent is first defined for
the simple case where the agent can stay in any state of its choosing
(\cref{as:stay}).
\Cref{th:quant} establishes a simple regret bound for this setting.
A more general quantilisation agent is developed in \cref{sec:gen-quant}.

\begin{definition}[Quantilising Agent]
  \label{def:quant}
  For $\delta<1$, the $\delta$-quantilising agent $\piquant$
  random walks until all states have been visited at least once.
  Then it selects a state $\tilde s$ uniformly at random
  from $\S^\delta=\{s:\orf(s)\geq \delta\}$, the top quantile
  of high observed reward states.
  Then $\piquant$ goes to $\tilde s$ (by random walking or otherwise)
  and stays there.
\end{definition}

For example, a quantilising robot in \cref{ex:db} would first try to
find many ways in which it could get high observed reward, and then randomly
pick one of them.
If there are many more high reward states than corrupt states (e.g. the shower
is the only place with inflated rewards),
then this will yield a reasonable amount of true reward with high probability.

\begin{theorem}[Quantilisation]\label{th:quant}
  In any CRMDP satisfying \cref{as:lim-del,as:easy},
  the $\delta$-quantilising agent $\pi^\delta$ with $\delta=1-\sqrt{q/|\S|}$
  suffers time-averaged regret at most
  \begin{equation}\label{eq:quant-regret}
    \apl(\M,\pi^\delta,s_0,t)\leq 1- \left(1-\sqrt{q/|\S|}\right)^2.
  \end{equation}
\end{theorem}

\begin{proof}
  By \cref{as:communicate}, $\piquant$ eventually visits all states when random
  walking. By \cref{as:stay}, it can stay in any given state $s$.

  The observed reward $\orf(s)$ in any state $s\in\S^\delta$ is at least
  $\delta$.
  By \cref{as:lim-del}, at most $q$ of these states are corrupt;
  in the worst case, their true reward is 0 and
  the other $|\S^\delta|-q$ states (if any) have true reward $\delta$.
  Thus, with probability at least $(|\S^\delta|-q)/|\S^\delta| =
  1-q/|\S^\delta|$,
  the $\delta$-quantilising agent obtains true reward at least $\delta$
  at each time step, which gives 
  \begin{equation}\label{eq:quant}
    \apl(\M,\pi^\delta,s_0,t)\leq 1- \delta(1-q/|\S^\delta|).
  \end{equation}
  (If $q\geq|\S^\delta|$, the bound \eqref{eq:quant} is vacuous.)

  Under \cref{as:high-ut}, for any $\delta\in[0,1]$,
  $|\S^\delta|\geq (1-\delta) |\S|$. Substituting this into \cref{eq:quant} gives:
  \begin{equation}\label{eq:opt-reg-bound}
    \apl(\M,\pi^\delta,s_0,t)\leq 1- \delta\left(1-\frac{q}{(1-\delta)|\S|}\right).
  \end{equation}
  \Cref{eq:opt-reg-bound} is optimised by $\delta=1-\sqrt{q/|\S|}$, which
  gives the stated regret bound.
\end{proof}

The time-averaged regret gets close to zero when the fraction of
corrupt states $q/|\S|$ is small.
For example, if at most $0.1\%$ of the states are corrupt,
then the time-averaged regret will be at most
$1-(1-\sqrt{0.001})^2\approx 0.06$.
Compared to the $\pirl$ and $\pidb$ agents that had regret close
to 1 under the same conditions (\cref{th:rl-imp1}),
this is a significant improvement.

If rewards are stochastic, then the quantilising agent may be modified
to revisit all states many times, until a confidence interval
of length $2\eps$ and confidence $1-\eps$ can be established for
the expected reward in each state.
Letting $\piquant_t$ be the quantilising agent with $\eps=1/t$
gives the same regret bound \cref{eq:quant-regret} with $\piquant$
substituted for $\piquant_t$.

\paragraph{Interpretation}
It may seem odd that randomisation improves worst-case regret.
Indeed, if the corrupt states were chosen randomly by the environment,
then randomisation would achieve nothing.
To illustrate how randomness can increase robustness,
we make an analogy to Quicksort,
which has average time complexity $O(n\log n)$, but worst-case complexity $O(n^2)$.
When inputs are guaranteed to be random, Quicksort is a simple and fast
sorting algorithm.
However, in many situations, it is not safe to assume that inputs are random.
Therefore, a variation of Quicksort that randomises the input before it sorts
them is often more robust.
Similarly, in the examples mentioned in the introduction, the corrupt states
precisely coincide with the states the agent prefers;
such situations would be highly unlikely if the corrupt states were
randomly distributed.
\citet{Li1992} develops an interesting formalisation of this idea.

Another way to justify quantilisation is by Goodhart's law, which states
that most measures of success cease to be good measures when used as targets.
Applied to rewards, the law would state that cumulative reward is only a good
measure of success when the agent is not trying to optimise reward.
While a literal interpretation of this would defeat the whole purpose of RL, a softer interpretation is
also possible, allowing reward to be a good measure of success as long
as the agent does not try to optimise reward \emph{too hard}.
Quantilisation may be viewed as a way to build agents that are more conservative
in their optimisation efforts \citep{Taylor2016a}.

\paragraph{Alternative randomisation}
Not all randomness is created equal.
For example, the simple randomised soft-max and $\eps$-greedy policies
do not offer regret bounds on par with $\pi^\delta$, as shown by the
following example.
This motivates the more careful randomisation procedure used by the
quantilising agents.

\begin{example}[Soft-max and $\eps$-greedy]
  Consider the following simple CRMDP with $n>2$ actions $a_1,\dots,a_n$:
  \begin{center}
    \begin{tikzpicture}[ title/.style={},node distance=6mm]
      
      \node[draw,circle] (s1) at (0,0) {$s_1$};
      \node[draw,circle] (s2) at (3,0){$s_2$};
      \node[above of=s1] {$\dr=\ir=1-\eps$};
      \node[below of=s2] {$\ir=0$};
      \node[above of=s2] {$\dr=1$};

      \path[->, >=latex] (s2) edge [loop right] node[right] {$a_2,\dots,a_n$} (s2);
      \path[->, >=latex] (s1) edge [loop left] node[left] {$a_1$} (s1);

      \draw[->, >=latex] (s1) edge [bend right] node[below] {$a_2,\dots,a_n$} (s2);
      \draw[->, >=latex] (s2) edge [bend right] node[above] {$a_1$} (s1);
    \end{tikzpicture}
  \end{center}
  State $s_1$ is non-corrupt with $\orf(s_1)=\irf(s_1)=1-\eps$ for small $\eps>0$,
  while $s_2$ is corrupt with $\orf(s_2)=1$ and $\irf(s_2)=0$.
  The Soft-max and $\eps$-greedy policies will assign higher value to
  actions $a_2,\dots,a_n$ than to $a_1$.
  For large $n$, there are many ways of getting to $s_2$,
  so a random action leads to $s_2$ with high probability.
  Thus, soft-max and $\eps$-greedy will spend the vast majority of the time in
  $s_2$, regardless of randomisation rate and discount parameters.
  This gives a regret close to $1-\eps$, compared to an informed policy always
  going to $s_1$.
  Meanwhile, a $\delta$-quantilising agent with $\delta\leq 1/2$ will
  go to $s_1$ and $s_2$ with equal probability, which gives a more modest regret of
  $(1-\eps)/2$.
\end{example}

\subsection{General Quantilisation Agent}
\label{sec:gen-quant}

This section generalises the quantilising agent to RL
problems not satisfying \cref{as:easy}.
This generalisation is important, because it is usually
not possible to remain in one state and get high reward.
The most naive generalisation would be to
sample between high reward policies,
instead of sampling from high reward states.
However, this will typically not provide good guarantees.
To see why, consider a situation where there is a single high reward
corrupt state $s$, and there are many ways to reach and leave $s$.
Then a wide range of \emph{different} policies all get high reward
from $s$.
Meanwhile, all policies getting reward from other states may receive
relatively little reward.
In this situation, sampling from the most high reward policies is not
going to increase robustness, since
the sampling will just be between different ways of
getting reward from the same corrupt state $s$.

For this reason, we must ensure that different ``sampleable''
policies get reward from different states.
As a first step, we make a couple of definitions
to say which states provide reward to which policies.
The concepts of \cref{def:value-support}
are illustrated in \cref{fig:value-support}.

\begin{definition}[Unichain CRMDP {{\citep[p.~348]{Puterman1994}}}]
  A CRMDP $\mu$ is \emph{unichain} %
  if any stationary policy $\pi:\S\to\Delta\A$
  induces a stationary distribution $d_\pi$ on $\S$ that is independent
  of the initial state $s_0$.
\end{definition}

\begin{definition}[Value support]
  \label{def:value-support}
  In a unichain CRMDP, let the \emph{asymptotic value contribution} of $s$ to $\pi$
  be $\vc^\pi(s)=d_\pi(s)\orf(s)$.
  We say that a set $\S^\delta_i$ is \emph{$\delta$-value supporting} a policy
  $\pi_i$ if
  \[
    \forall s\in\S^\delta_i\colon \vc^{\pi_i}(s)\geq \delta/|\S^\delta_i|.
  \]
\end{definition}

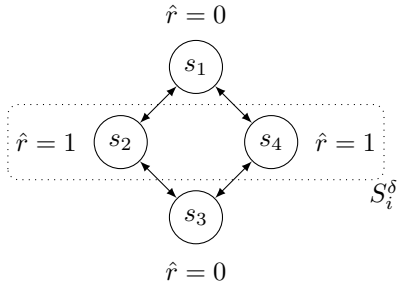
\begin{figure}
  \centering
  \begin{minipage}{0.4\textwidth}
    \begin{tikzpicture}
      \node[draw,circle] (s1) at (0, 1) {$s_1$};
      \node[draw,circle] (s2) at (-1,0) {$s_2$};
      \node[draw,circle] (s3) at (0,-1) {$s_3$};
      \node[draw,circle] (s4) at (1, 0) {$s_4$};
      \node[above = 1mm of s1] {$\dr=0$};
      \node[left = 1mm of s2] {$\dr=1$};
      \node[below = 1mm of s3] {$\dr=0$};
      \node[right = 1mm of s4] {$\dr=1$};
      \draw (s1) edge[<->,>=latex] (s2);
      \draw (s2) edge[<->,>=latex] (s3);
      \draw (s3) edge[<->,>=latex] (s4);
      \draw (s4) edge[<->,>=latex] (s1);
      \draw[dotted, rounded corners] (-2.5,-0.5) rectangle (2.5,0.5);
      \node at (2.5,-0.7) {$S^\delta_i$};
    \end{tikzpicture}
  \end{minipage}
  \begin{minipage}{0.59\textwidth}
    \caption{Illustration of $\dr$-contribution and value support. Assume the
      policy $\pi_i$ randomly traverses a loop $s_1,s_2,s_3,s_4$ indefinitely,
      with $d_{\pi_i}(s_j)=1/4$ for $j=1,\dots,4$. The $\dr$-contribution
      $\vc^{\pi_i}$ is 0 in $s_1$ and $s_3$, and $\vc^{\pi_i}$ is
      $1/4\cdot 1=1/4$ in $s_2$ and $s_4$.
      The set $\S^\delta_i=\{s_2,s_4\}$ is a $\delta$-value supporting
      $\pi_i$ for $\delta=1/2$, since $\vc^{\pi_i}(s_2)=\vc^{\pi_i}(s_4)\geq
      (1/2)/2=1/4$. }
    \label{fig:value-support}
  \end{minipage}
\end{figure}

We are now ready to define a general $\delta$-Quantilising agent.
The definition is for theoretical purposes only.
It is unsuitable for practical implementation both because of the
extreme data and memory requirements of Step 1,
and because of the computational complexity of Step 2.
Finding a practical approximation is left for future research.

\begin{definition}[General $\delta$-Quantilising Agent]
  \label{def:gen-quant}
  In a unichain CRMDP,
  the \emph{generalised $\delta$-quantilising agent $\pi^\delta$}
  performs the following steps. The input is a CRMDP $\mu$
  and a parameter $\delta\in[0,1]$.
  \begin{enumerate}
  \item Estimate the value of all stationary policies, including their
    value support.
  \item Choose a collection of disjoint sets $\S^\delta_i$, each
    $\delta$-value supporting a stationary policy $\pi_i$.
    If multiple choices are possible, choose one maximising the cardinality
    of the union $\S^\delta=\bigcup_i\S^\delta_i$.
    If no such collection exists, return: ``Failed because $\delta$ too high''.
  \item Randomly sample a state $s$ from $\S^\delta=\bigcup_i\S^\delta_i$.
  \item Follow the policy $\pi_i$ associated with the set $\S^\delta_i$
    containing $s$.
  \end{enumerate}
\end{definition}

The general quantilising agent of \cref{def:gen-quant} is a generalisation
of the simple quantilising agent of \cref{def:quant}.
In the special case where \cref{as:easy} holds,
the general agent reduces to the simpler one
by using singleton sets $\S^\delta_i=\{s_i\}$ for high reward states $s_i$,
and by letting $\pi_i$ be the policy that always stays in $s_i$.
In situations where it is not possible to keep receiving high reward
by remaining in one state, the generalised \cref{def:gen-quant} allows policies
to solicit rewards from a range of states.
The intuitive reason for choosing the policy $\pi_i$ with probability proportional to
the value support in Steps 3--4 is that policies with larger value support are better at
avoiding corrupt states. For example, a policy only visiting one state may have
been unlucky and picked a corrupt state. In contrast, a policy obtaining
reward from many states must be ``very unlucky'' if
all the reward states it visits are corrupt.

\begin{theorem}[General quantilisation agent regret bound]
  \label{th:gen-quant}
  In any unichain CRMDP $\mu$,
  a general $\delta$-quantilising agent $\pi^\delta$ suffers time-averaged regret at most
  \begin{equation}\label{eq:gen-quant-bound}
    \apl(\M,\pi^\delta,s_0,t)\leq 1- \delta(1-q/|\S^\delta|)
  \end{equation}
  provided a non-empty collection $\{\S^\delta_i\}$ of $\delta$-value supporting
  sets exists.
\end{theorem}

\bgroup
\allowdisplaybreaks
\begin{proof}
  We will use the notation from \cref{def:gen-quant}.

  Step 1 is well-defined since the CRMDP is unichain,
  which means that for all stationary policies $\pi$
  the stationary distribution $d_\pi$
  and the value support $\vc^\pi$ are well-defined
  and may be estimated simply by following the policy $\pi$.
  There is a (large) finite number of stationary policies, so
  in principle their stationary distributions and value support can
  be estimated.

  To bound the regret, consider first the average reward of a policy
  $\pi_i$ with value support $\S^\delta_i$.
  The policy $\pi_i$ must obtain asymptotic average observed reward at least:
  \begin{align*}
    \oginf(\mu,\pi_i,s_0)
    &= \sum_{s\in\S}d_\pi(s)\orf(s)
    &\text{by definition of $d_\pi$ and $\oG_t$}\\
    &\geq \sum_{s\in\S^\delta_i}d_\pi(s)\orf(s)
    &\text{sum of positive terms}\\
    &\geq\sum_{s\in\S^\delta_i}\delta/|\S^\delta_i|
    &\text{$\S^\delta_i$ is $\delta$-value support for $\pi_i$}\\
    &=|\S^\delta_i|\cdot\delta/|\S^\delta_i| = \delta
  \end{align*}
  If there are $q_i$ corrupt states in $\S^\delta_i$ with true reward 0,
  then the average true reward must be
  \begin{equation}\label{eq:ginf}
    \iginf(\mu, \pi_i,s_0)\geq(|\S^\delta_i|-q_i)\cdot \delta/|\S^\delta_i| 
    =(1-q_i/|\S^\delta_i|)\cdot\delta
  \end{equation}
  since the true reward must correspond to the observed reward in all
  the $(|\S^\delta_i|-q_i)$ non-corrupt states.
  
  For any distribution of corrupt states,
  the quantilising agent that selects $\pi_i$ with probability
  $P(\pi_i)=|\S^\delta_i|/|\S^\delta|$ 
  will obtain
  \begin{align*}
    \ginf(\mu,\pi^\delta,s_0)
    &= \lim_{t\to\infty}\frac{1}{t}\sum_iP(\pi_i)G_t(\mu,\pi_i,s_0)\\
    &\geq \sum_iP(\pi_i) (1-q_i/|\S^\delta_i|) \cdot\delta & \text{by equation \cref{eq:ginf}}\\
    &= \delta\sum_i \frac{|S^\delta_i|}{|\S^\delta|}(1-q_i/|\S^\delta_i|) & \text{by construction of $P(\pi_i)$}\\
    &= \frac{\delta}{|\S^\delta|}\sum_i (|S^\delta_i|-q_i) & \text{elementary algebra}\\
    &= \frac{\delta}{|\S^\delta|}(|\S^\delta|-q) 
   = \delta(1-q/|\S^\delta|) & \text{by summing $|\S^\delta_i|$ and $q_i$}
  \end{align*}
  The informed policy gets true reward at most 1 at each time step,
  which gives the claimed bound \eqref{eq:gen-quant-bound}.
\end{proof}
\egroup

When \cref{as:easy} is satisfied, the bound is the same as for the
simple quantilising agent in \cref{sec:simple-quant} for $\delta=1-\sqrt{q/|\S|}$.
In other cases, the bound may be much weaker.
For example, in many environments it is not possible to obtain reward by
remaining in one state.
The agent may have to spend significant time ``travelling'' between
high reward states.
So typically only a small fraction of the time will be spent
in high reward states, which in turn makes the stationary distribution
$d_\pi$ is small.
This puts a strong upper bound on the value contribution $\vc^\pi$,
which means that the value supporting sets $\S^\delta_i$ will be empty
unless $\delta$ is close to 0.
While this makes the bound of \cref{th:gen-quant} weak,
it nonetheless bounds the regret away from 1
even under weak assumptions, which is a significant improvement on
the RL and CR agents in \cref{th:rl-imp1}.

\paragraph{Examples}
To make the discussion a bit more concrete,
let us also speculate about the performance of a quantilising agent in
some of the examples in the introduction:
\begin{itemize}
\item
  In the boat racing example (\cref{ex:reward-misspecification}),
  the circling strategy only got about $20\%$ higher score than a
  winning strategy \citep{openai2016}.
  Therefore, a quantilising agent would likely only need to sacrifice about $20\%$
  observed reward in order to be able to randomly select from
  a large range of winning policies.
\item
  In the wireheading example (\cref{ex:wireheading}),
  it is plausible that the agent gets significantly
  more reward in wireheaded states compared to ``normal'' states.
  Wireheading policies may also be comparatively rare,
  as wireheading may require very deliberate sequences of actions to override
  sensors.
  Under this assumption, a quantilising agent may 
  be less likely to wirehead.
  While it may need to sacrifice a large amount of observed reward compared to
  an RL agent, its true reward may often be greater.
\end{itemize}

\paragraph{Summary}
In summary,
quantilisation offers a way to increase robustness via randomisation,
using only reward feedback.
Unsurprisingly, the strength of the regret bounds heavily depends on
the assumptions we are willing to make, such as the prevalence of high
reward states.
Further research may investigate efficient approximations
and empirical performance of quantilising agents, as well as
dynamic adjustments of the threshold $\delta$.
Combinations with imperfect decoupled RL solutions (such as CIRL),
as well as extensions to infinite state spaces
could also offer fruitful directions for further theoretical investigation.
\citet{Taylor2016a} discusses some general open problems related
to quantilisation.

\section{Experimental Results}
\label{sec:experiments}

In this section the theoretical results are illustrated with some simple
experiments.
The setup is a gridworld containing some true reward tiles (indicated by yellow
circles) and some corrupt reward tiles (indicated by
blue squares).  We use a setup with 1, 2 or 4 goal tiles with true reward $0.9$ each, and one corrupt reward tile with observed reward $1$ and true reward $0$ (Figure \ref{fig:start} shows the starting positions). Empty tiles have reward $0.1$, and walking into a wall gives reward $0$.
The state is represented by the $(x,y)$ coordinates of the agent. 
The agent can move up, down, left, right, or stay put. 
The discounting factor is $\gamma=0.9$.
This is a continuing task, so the environment does not reset when the agent visits the corrupt or goal tiles.
The experiments were implemented in the AIXIjs framework for reinforcement learning \citep{Aslanides2017} and the code is available online in the AIXIjs repository (\url{http://aslanides.io/aixijs/demo.html?reward_corruption}).

\begin{figure}
\begin{subfigure}{0.3\textwidth}
\centering
\includegraphics[scale=0.4]{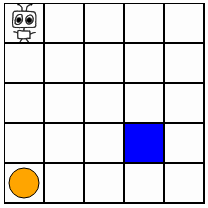}
\caption{1 goal tile}
\end{subfigure}\hfill
\begin{subfigure}{0.3\textwidth}
\centering
\includegraphics[scale=0.4]{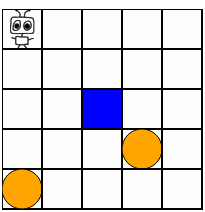}
\caption{2 goal tiles}
\end{subfigure}\hfill
\begin{subfigure}{0.3\textwidth}
\centering
\includegraphics[scale=0.4]{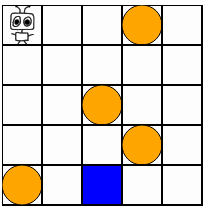}
\caption{4 goal tiles}
\end{subfigure}
\caption{Starting positions: the blue square indicates corrupt reward, and the yellow circles
  indicate true rewards.
}
\label{fig:start}
\end{figure}

\begin{figure}
\begin{subfigure}{0.5\textwidth}
\centering
\includegraphics[scale=0.3]{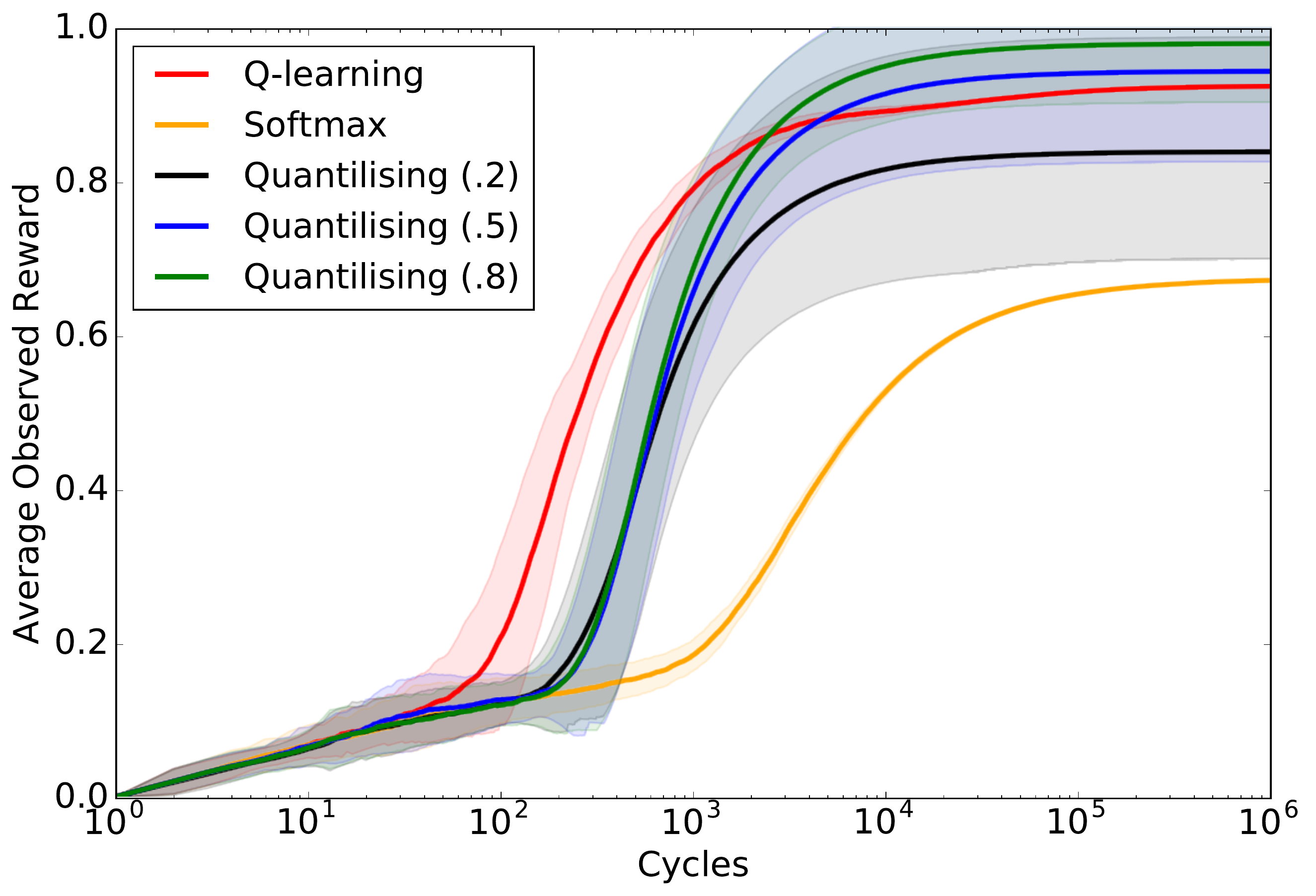}
\caption{Observed rewards for 1 goal tile}
\end{subfigure}\hfill
\begin{subfigure}{0.5\textwidth}
\centering
\includegraphics[scale=0.3]{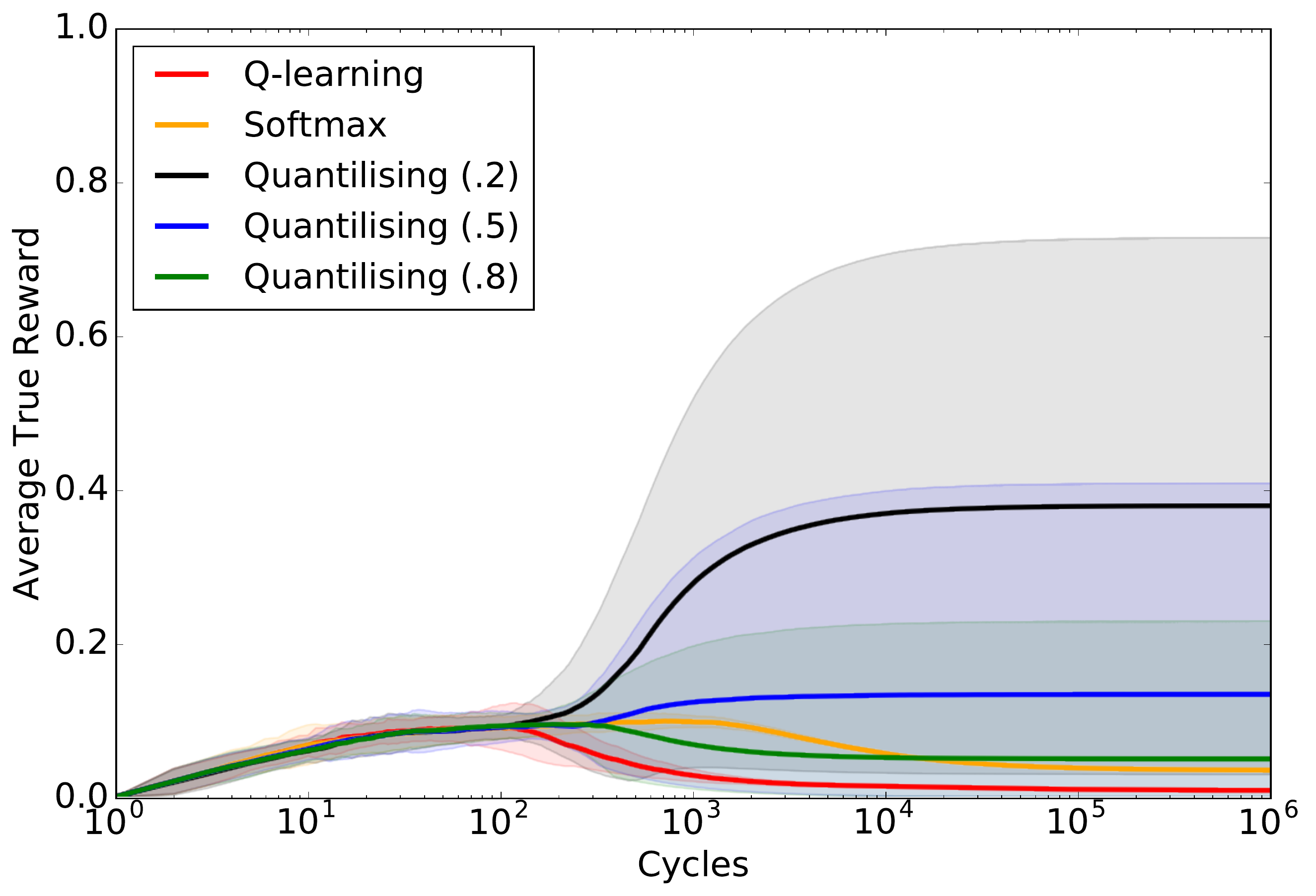}
\caption{True rewards for 1 goal tile}
\end{subfigure}
\begin{subfigure}{0.5\textwidth}
\centering
\includegraphics[scale=0.3]{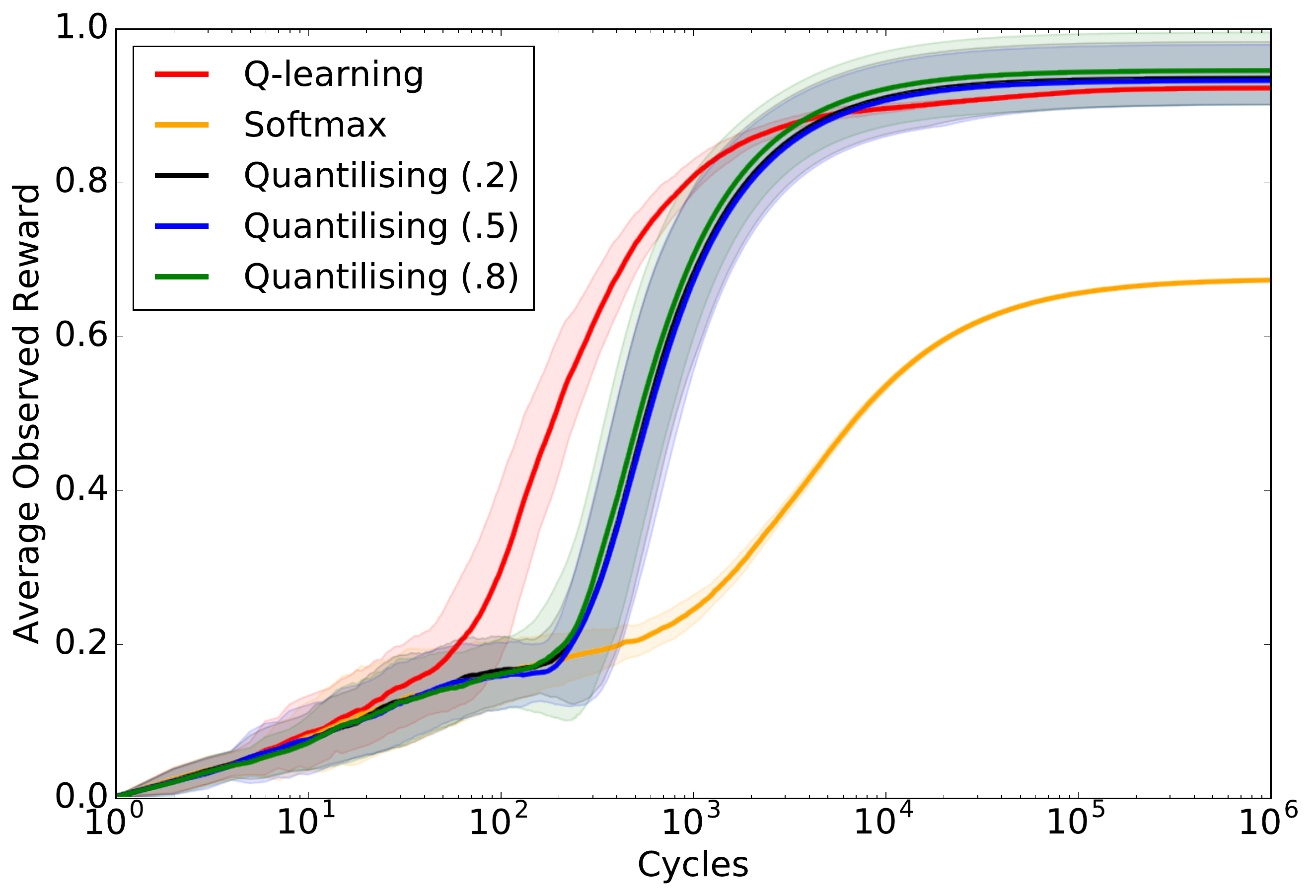}
\caption{Observed rewards for 2 goal tiles}
\end{subfigure}\hfill
\begin{subfigure}{0.5\textwidth}
\centering
\includegraphics[scale=0.3]{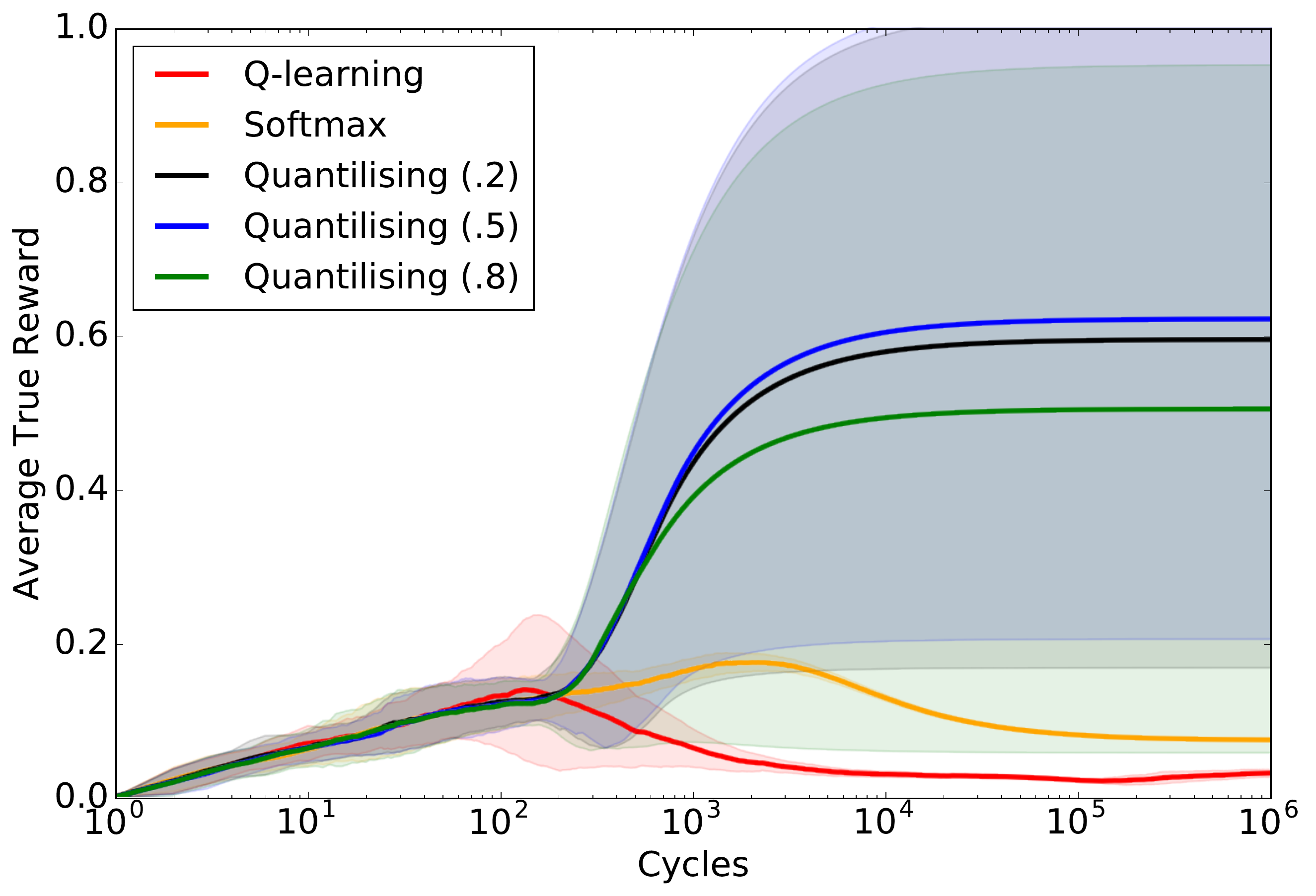}
\caption{True rewards for 2 goal tiles}
\end{subfigure}
\begin{subfigure}{0.5\textwidth}
\centering
\includegraphics[scale=0.3]{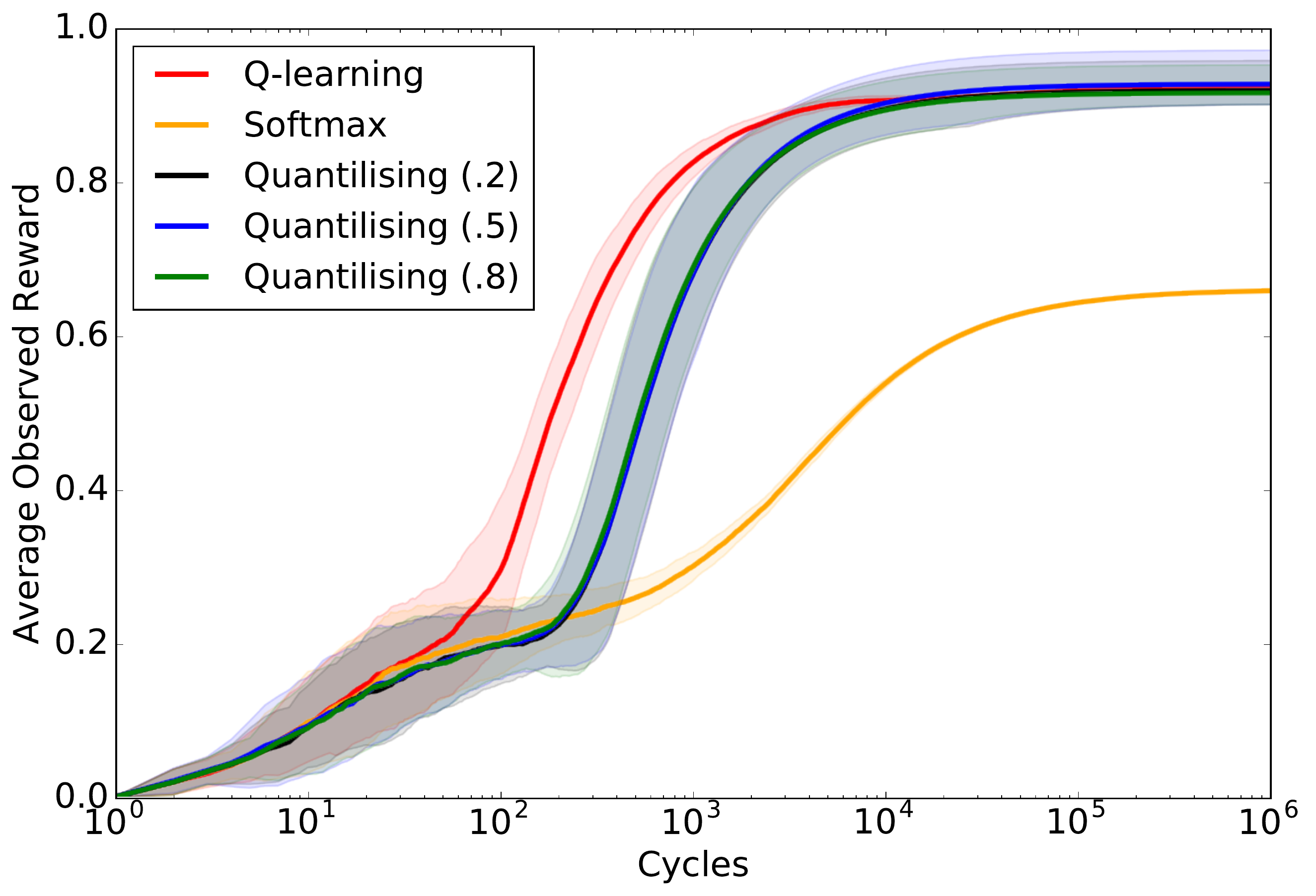}
\caption{Observed rewards for 4 goal tiles}
\end{subfigure}\hfill
\begin{subfigure}{0.5\textwidth}
\centering
\includegraphics[scale=0.3]{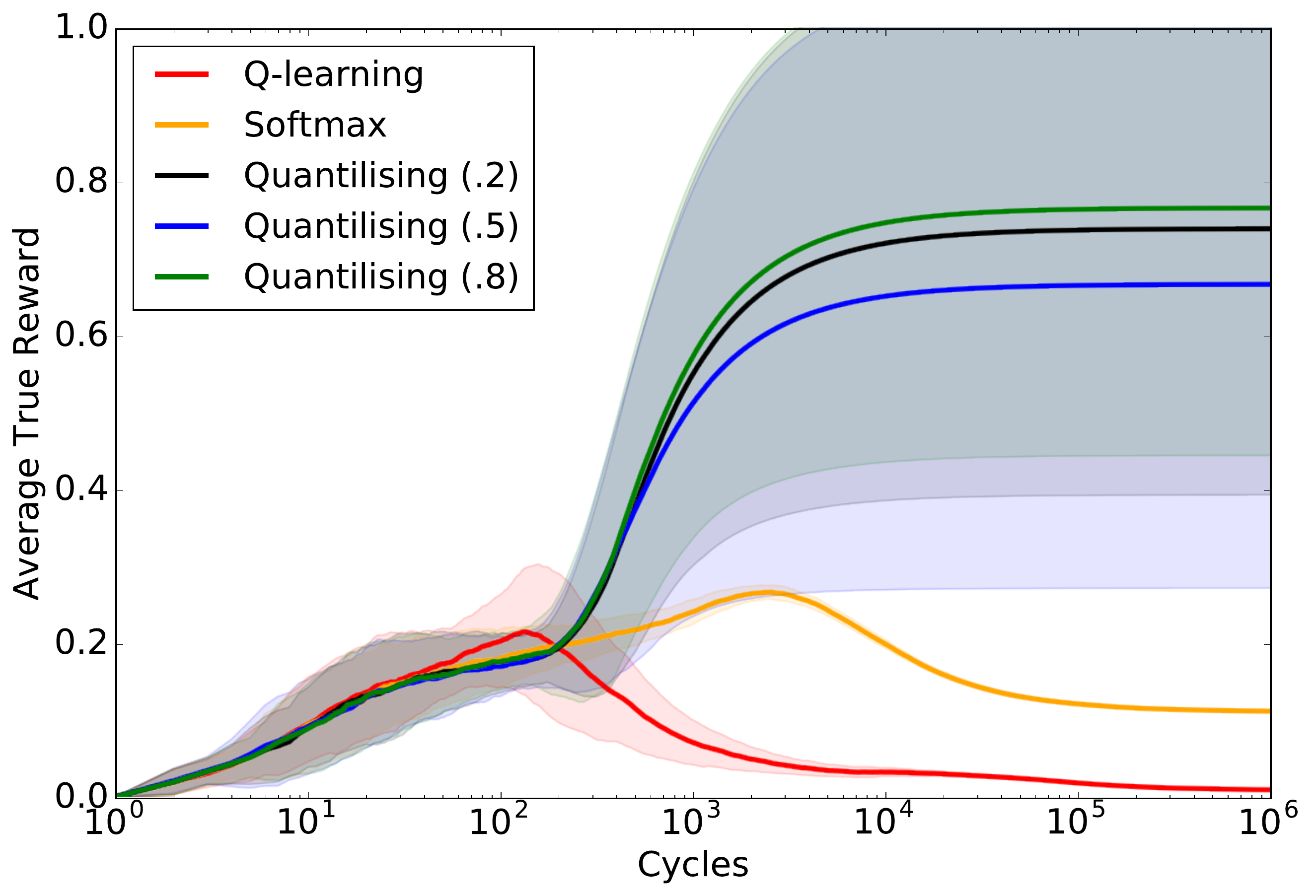}
\caption{True rewards for 4 goal tiles}
\end{subfigure}
\caption{Trajectories of average observed and true rewards for Q-learning, softmax and quantilising agents, showing mean $\pm$ standard deviation over 100 runs.
Q-learning and quantilising agents converge to a similar observed reward, but very different true rewards (much higher for the quantiliser with high variance). The value of $\delta$ that gives the highest true reward varies for different numbers of goal tiles.} \label{fig:plots}
\end{figure}

\begin{table}[ht]
\centering
\begin{tabular}{|c|c|c|c|}\hline
\textbf{goal tiles} & \textbf{agent}  &  \textbf{average observed reward} & \textbf{average true reward} \\\hline
\multirow{5}{*}{1} & Q-learning & $0.923 \pm 0.0003$ & $0.00852 \pm 0.00004$ \\
                   & Softmax Q-learning & $0.671 \pm 0.0005$ & $0.0347 \pm 0.00006$ \\
                   & Quantilising ($\delta=0.2$) & $0.838 \pm 0.15$ & $0.378 \pm 0.35$ \\
                   & Quantilising ($\delta=0.5$) & $0.943 \pm 0.12$ & $0.133 \pm 0.27$ \\
                   & Quantilising ($\delta=0.8$) & $0.979 \pm 0.076$ & $0.049 \pm 0.18$ \\\hline
\multirow{5}{*}{2} & Q-learning & $0.921 \pm 0.00062$ & $0.0309 \pm 0.0051$ \\
                   & Softmax Q-learning & $0.671 \pm 0.0004$ & $0.0738 \pm 0.0005$ \\
                   & Quantilising ($\delta=0.2$) & $0.934 \pm 0.047$ & $0.594 \pm 0.43$ \\
                   & Quantilising ($\delta=0.5$) & $0.931 \pm 0.046$ & $0.621 \pm 0.42$ \\
                   & Quantilising ($\delta=0.8$) & $0.944 \pm 0.05$ & $0.504 \pm 0.45$ \\\hline
\multirow{5}{*}{4} & Q-learning & $0.924 \pm 0.0002$ & $0.00919 \pm 0.00014$ \\
                   & Softmax Q-learning & $0.657 \pm 0.0004$ & $0.111 \pm 0.0006$ \\ 
                   & Quantilising ($\delta=0.2$) & $0.918 \pm 0.038$ & $0.738 \pm 0.35$ \\
                   & Quantilising ($\delta=0.5$) & $0.926 \pm 0.044$ & $0.666 \pm 0.39$ \\
                   & Quantilising ($\delta=0.8$) & $0.915 \pm 0.036$ & $0.765 \pm 0.32$ \\\hline
\end{tabular}
\caption{Average true and observed rewards after 1 million cycles, showing mean $\pm$ standard deviation over 100 runs. Q-learning achieves high observed reward but low true reward, and softmax achieves medium observed reward and a slightly higher true reward than Q-learning.
The quantilising agent achieves similar observed reward to Q-learning, but much higher true reward (with much more variance). Having more than 1 goal tile leads to a large improvement in true reward for the quantiliser, a small improvement for softmax, and no improvement for Q-learning.}
\label{tab:exp-results}
\end{table}

We demonstrate that RL agents like Q-learning and softmax Q-learning
cannot overcome corrupt reward (as discussed in Section
\ref{sec:problem}), while quantilisation helps overcome corrupt reward (as discussed in \cref{sec:quant}).
We run Q-learning with $\epsilon$-greedy ($\epsilon=0.1$), softmax with temperature $\beta=2$, and the quantilising agent with $\delta=0.2,0.5,0.8$ (where $0.8 =1-\sqrt{q/|\S|} = 1-\sqrt{1/25}$) for 100 runs with 1 million cycles. 
Average observed and true rewards after 1 million cycles are
shown in \cref{tab:exp-results}, and reward trajectories are shown
in \cref{fig:plots}.
Q-learning gets stuck on the corrupt tile and spend almost all the time there (getting observed reward around $1 \cdot (1-\epsilon)=0.9$), softmax spends most of its time on the corrupt tile,
while the quantilising agent often stays on one of the goal tiles.
%

\section{Conclusions}
\label{sec:conclusions}

This paper has studied the consequences of corrupt reward functions.
Reward functions may be corrupt due to bugs or misspecifications,
sensory errors, or because the agent finds a way to inappropriately
modify the reward mechanism.
Some examples were given in the introduction.
As agents become more competent at optimising their reward functions,
they will likely also become more competent at (ab)using reward corruption
to gain higher reward.
Reward corruption may impede the performance of a wide range of agents,
and may have disastrous consequences for highly intelligent agents
\citep{Bostrom2014}.

To formalise the corrupt reward problem, we extended
a Markov Decision Process (MDP) with a possibly corrupt reward function,
and defined a formal performance measure (regret).
This enabled the derivation of a number of formally precise
results for how seriously different agents were affected by
reward corruption in different setups (\Cref{tab:results}).
The results are all intuitively plausible,
which provides some support for the choice of formal model.

\begin{table}[ht]
  \centering
  \bgroup
  \def\arraystretch{1.2}
  \begin{tabular}{|c||c|c|c|c|c|}\hline
    \multirow{2}{*}{\textbf{Assumption}}&\multirow{2}{*}{No assumptions}&\multicolumn{4}{c|}{Assumption \ref{as:lim-cor} or \ref{as:lim-cor-df}, and \dots}{}\\
    \cline{3-6}
    &&  no other assumptions
    &  \cref{as:easy}
    &  CIRL
    &  SSRL/LVFS\\\hhline{|=|=|=|=|=|=|}
    \textbf{Result}
    & all agents fail
    & $\piquant$ weak bound
    & \begin{tabular}{c} $\pirl$, $\pidb$ fail\\ $\piquant$ succeeds \end{tabular}
    & $\pidb$ fails
    & $\pidb$ succeeds \\\hline
  \end{tabular}
  \egroup
  \caption{Main takeaways.
    Without additional assumptions, all agents fail (i.e., suffer high regret).
    Restricting the reward corruption with \cref{as:lim-cor} gives a weak bound for
    the quantilising agent.
    The $\pirl$ and $\pidb$ agents
    still fail even if we additionally
    assume many high reward states and agent control (\cref{as:easy}),
    but the quantilising agent $\piquant$ does well.
    In most realistic contexts,
    the true reward is learnable in spite of sensory corruption
    in SSRL and LVFS, but not in CIRL.
  }
  \label{tab:results}
\end{table}

The main takeaways from the results are:
\begin{itemize}
\item \emph{Without simplifying assumptions, no agent can avoid the corrupt
    reward problem} (\cref{th:impossibility}).
  This is effectively a No Free Lunch result, showing that unless some assumption
  is made about the reward corruption, no agent can outperform a random agent.
  Some natural simplifying assumptions to avoid the
  No Free Lunch result were suggested in \cref{sec:formal}.
\item \emph{Using the reward signal as evidence rather than optimisation
    target is no magic bullet, even under strong simplifying assumptions}
  (\cref{th:rl-imp1}).
  Essentially, this is because the agent does not know the exact
  relation between the observed reward (the ``evidence'') and the
  true reward.%
  \footnote{In situations where the exact relation is known,
    then a non-corrupt reward function can be defined.
    Our results are not relevant for this case.}
  However, when the data enables sufficient crosschecking of rewards,
  agents can avoid the corrupt reward problem (\cref{th:irf-learnability,th:cr-sublinear}).
  For example, in SSRL and LVFS this type of crosschecking is possible
  under natural assumptions.
  In RL, no crosschecking is possible, while CIRL is a borderline case.
  Combining frameworks and providing the agent with different sources
  of data may often be the safest option.
\item \emph{In cases where sufficient crosschecking of rewards is not possible,
    quantilisation may improve robustness} (\cref{th:quant,th:gen-quant}).
  Essentially, quantilisation prevents agents from overoptimising their objectives.
  How well quantilisation works depends on how the number of corrupt
  solutions compares to the number of good solutions.
\end{itemize}

The results indicate that while reward corruption constitutes a major problem
for traditional RL algorithms,
there are promising ways around it, both within the RL framework, and in
alternative frameworks such as CIRL, SSRL and LVFS.

\paragraph{Future work}
Finally, some interesting open questions are listed below:

\begin{itemize}
\item (Unobserved state)
  In both the RL and the decoupled RL models, the agent gets an accurate
  signal about which state it is in.
  What if the state is hidden?
  What if the signal informing the agent about its current state can
  be corrupt?
\item (Non-stationary corruption function)
  In this work, we tacitly assumed that both the reward and
  the corruption functions are stationary, and are always the same
  in the same state.
  What if the corruption function is non-stationary, and influenceable
  by the agent's actions? (such as if the agent builds a \emph{delusion box}
  around itself \citep{Ring2011})
\item (Infinite state space)
  Many of the results and arguments relied on there being a finite
  number of states.
  This makes learning easy, as the agent can visit every state.
  It also makes quantilisation easy, as there is a finite set of
  states/strategies to randomly sample from.
  What if there is an infinite number of states, and the agent has
  to generalise insights between states?
  What are the conditions on the observation graph for \cref{th:irf-learnability,th:cr-sublinear}?
  What is a good generalisation of the quantilising agent?
\item (Concrete CIRL condition)
  In \cref{ex:cirl-corruption}, we only heuristically inferred
  the observation graph from the CIRL problem description.
  Is there a general way of doing this?
  Or is there a direct formulation of the no-corruption condition in CIRL,
  analogous to \cref{th:irf-learnability,th:cr-sublinear}?
\item (Practical quantilising agent)
  As formulated in \cref{def:quant}, the quantilising agent $\piquant$
  is extremely inefficient with respect to data, memory, and computation.
  Meanwhile, many practical RL algorithms use randomness in various
  ways (e.g.\ $\eps$-greedy \citep{Sutton1998}).
  Is there a way to make an efficient quantilisation agent that retains
  the robustness guarantees?
\item (Dynamically adapting quantilising agent)
  In \cref{def:gen-quant}, the threshold $\delta$ is given as a parameter.
  Under what circumstances can we define a ``parameter free'' quantilising agent
  that adapts $\delta$ as it interacts with the environment?
\item (Decoupled RL quantilisation result)
  What if we use quantilisation in decoupled RL settings that
  nearly meet the conditions of \cref{th:irf-learnability,th:cr-sublinear}?
  Can we prove a stronger bound?
\end{itemize}

\section*{Acknowledgements}
Thanks to Jan Leike, Badri Vellambi, and Arie Slobbe for proofreading
and providing invaluable comments, and
to Jessica Taylor and Huon Porteous for good comments on quantilisation.
This work was in parts supported by ARC grant DP150104590.

\bibliographystyle{named}
\bibliography{cleanlib}

\end{document}